\documentclass{article}

\PassOptionsToPackage{numbers, compress}{natbib}
\usepackage[preprint]{neurips_2025}

\usepackage[utf8]{inputenc} 
\usepackage[T1]{fontenc}    
\usepackage{hyperref}       
\usepackage{url}            
\usepackage{booktabs}       
\usepackage{amsfonts}       
\usepackage{nicefrac}       
\usepackage{microtype}      
\usepackage{xcolor}         

\usepackage{natbib}
\usepackage{caption}
\usepackage{graphicx}
\usepackage{amsmath}
\usepackage{amsthm}
\usepackage{subcaption}
\usepackage{amssymb}
\usepackage{mathtools}
\usepackage{bm}
\usepackage{xspace}
\usepackage[inline]{enumitem}
\usepackage{comment}
\usepackage{threeparttable}
\usepackage{tikz}
\usepackage{multirow}
\usepackage{multicol}
\usepackage[linesnumbered, ruled]{algorithm2e}

\title{Low-Rank Curvature for Zeroth-Order Optimization in LLM Fine-Tuning}

\author{%
  Hyunseok Seung$^{1}$ \quad
  Jaewoo Lee$^{2}$ \quad
  Hyunsuk Ko$^{3}$ \\
  $^{1}$University of Wisconsin -- Madison \quad
  $^{2}$University of Georgia \quad
  $^{3}$Hanyang University \\
  \texttt{hseung2@wisc.edu}, \ \texttt{jaewoo.lee@uga.edu}, \ \texttt{hyunsuk@hanyang.ac.kr}
}

\let\en=\ensuremath

\newcommand{\dd}[1]{\mathop{}\!\mathrm{d} #1}

\DeclareMathOperator{\E}{\mathbb{E}}          

\DeclareMathOperator*{\argmin}{arg\,min\,}

\DeclarePairedDelimiter{\norm}{\lVert}{\rVert}
\DeclarePairedDelimiter{\abs}{\lvert}{\rvert}

\DeclareMathOperator{\Trace}{tr}
\DeclareMathOperator{\Det}{det}

\DeclareMathOperator{\vect}{vec}                

\renewcommand{\vec}[1]{\en{\bm{\mathrm{#1}}}}
\newcommand{\tvec}[1]{\en{\widetilde{\vec{#1}}}}
\newcommand{\mat}[1]{\en{{\bm{\mathrm{#1}}}}}
\newcommand{\grad}[0]{\en{\nabla}}

\newcommand{\R}[0]{\mathbb{R}}

\newcommand{\pdv}[2]{\frac{\partial #1}{\partial #2}}

\renewcommand{\th}[0]{\textsuperscript{th}\xspace}

\newcommand{\blue}[1]{\textcolor{blue}{#1}}

\newcommand{\loren}[0]{\textsc{LOREN}\xspace}

\theoremstyle{plain}
\newtheorem{theorem}{Theorem}[section]
\newtheorem{proposition}[theorem]{Proposition}
\newtheorem{lemma}[theorem]{Lemma}

\theoremstyle{definition}
\newtheorem{definition}[theorem]{Definition}
\newtheorem{assumption}[theorem]{Assumption}

\theoremstyle{remark}

\usepackage{pifont}

\begin{document}

\maketitle

\begin{abstract}
We introduce \loren, a curvature-aware zeroth-order (ZO) optimization method
for fine-tuning large language models (LLMs). Existing ZO methods,
which estimate gradients via finite differences using random
perturbations, often suffer from high variance and suboptimal search
directions.
Our approach addresses these challenges by: (i) reformulating the problem of gradient
preconditioning as that of adaptively estimating an 
anisotropic perturbation distribution for gradient estimation,
(ii) capturing curvature through a low-rank block diagonal preconditioner using the framework of natural evolution strategies, and
(iii) applying a REINFORCE leave-one-out (RLOO) gradient estimator to reduce 
variance. 
Experiments on standard LLM benchmarks show that our method outperforms state-of-the-art ZO methods by achieving higher accuracy and faster convergence, while cutting peak memory usage by up to 27.3\% compared with MeZO-Adam.


\end{abstract}

Code is available at \href{https://github.com/hseung88/loren}{https://github.com/hseung88/loren}.
\section{Introduction}
\label{sec:introduction}
Fine-tuning large language models (LLMs) with first-order (FO) methods such as
SGD~\cite{robbins1951Stochastic} and
AdamW~\cite{kingma2015adam,Loshchilov2019DecoupledWD}
incurs significant memory overhead primarily due to 
gradient computations during backpropagation.
To address this limitation, there has been
renewed interest in developing zeroth-order (ZO) optimization
methods for LLM fine-tuning.
%
%
%
Recent ZO optimizers, such as MeZO~\cite{malladi2023finetuning},
estimate gradient using only forward-pass evaluations of the model,
eliminating the need to store intermediate activations or perform
backpropagation for gradient computation, thereby significantly
reducing memory requirements. The low memory footprint makes ZO
optimizers particularly appealing for LLM fine-tuning and recent
studies~\cite{malladi2023finetuning,chen2025enhancing} have shown promising results.
%
%

Despite the memory efficiency, existing ZO optimizers exhibit slow
convergence rates
due to two fundamental limitations.
First, the finite-difference gradient estimators employed in ZO methods
suffer from high variance, 
particularly in high-dimensional 
stochastic settings. This high variance leads to noisy gradient
approximations, resulting in unstable parameter updates 
and 
degraded optimization performance~\cite{Nesterov2017RandomGM,Ghadimi2013StochasticFA}.
In the absence of variance reduction techniques, the sample complexity
measured in terms of function evaluations scales poorly with model
dimensionality~\cite{Duchi2013OptimalRF,Nesterov2017RandomGM,gautam2024variancereduced}.
%
Second, existing ZO optimizers are agnostic to the anisotropic
curvature of loss landscapes in LLMs, i.e., they fail to adapt
to curvature heterogeneity across different weights and layers. This
lack of curvature awareness can lead to 
optimization inefficiencies (e.g., oscillations in high-curvature directions or stagnation
along nearly flat directions) and may even result in convergence to saddle
points~\cite{zhao2025secondorder}. 

\begin{figure*}[tb]
    \centering
    \begin{subfigure}{0.37\textwidth}
        \centering
        \includegraphics[width=\textwidth]{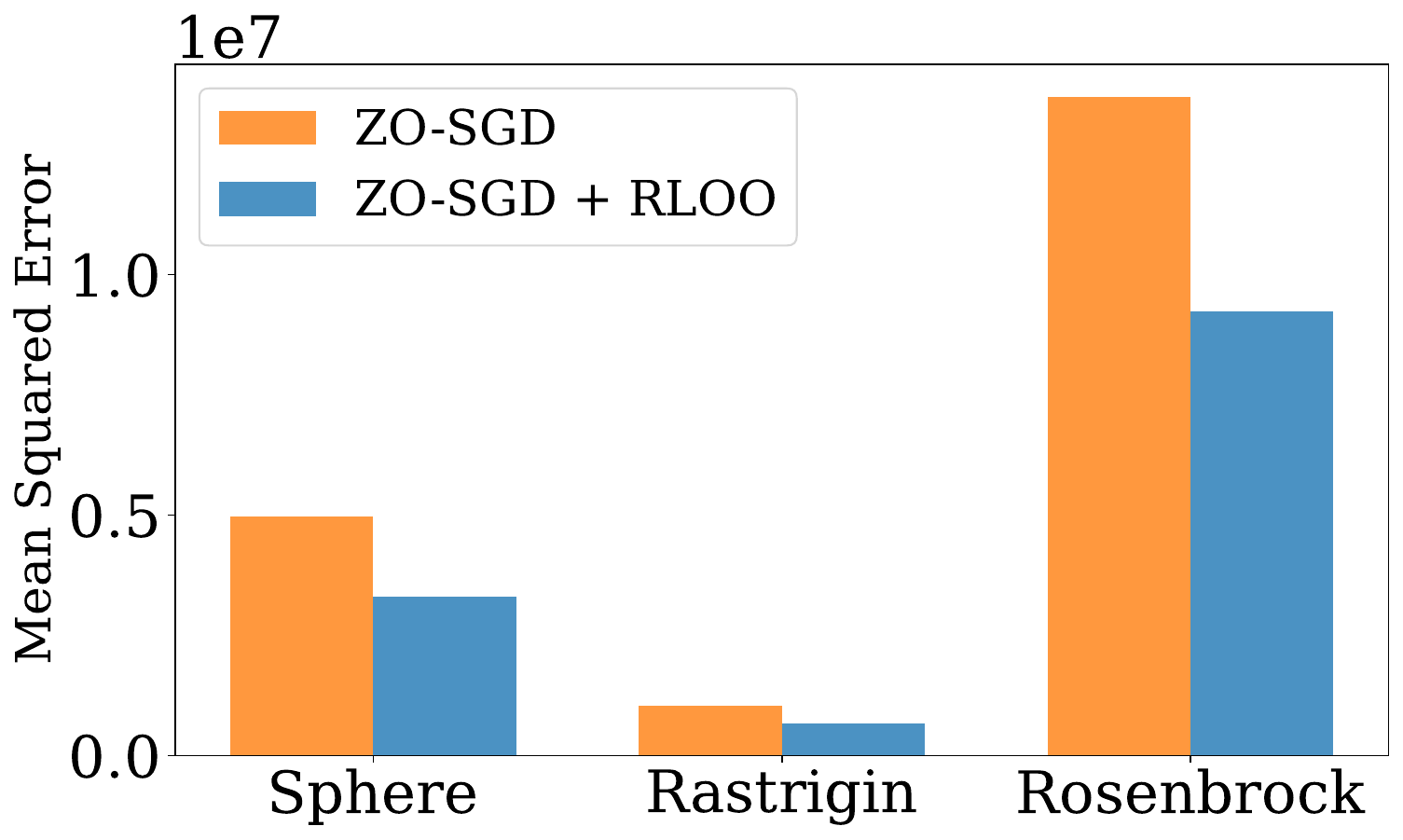}
        \subcaption{}
        \label{fig:optim_demo_a}
    \end{subfigure}
    \qquad
    \begin{subfigure}{0.34\textwidth}
        \centering
        \includegraphics[width=\textwidth]{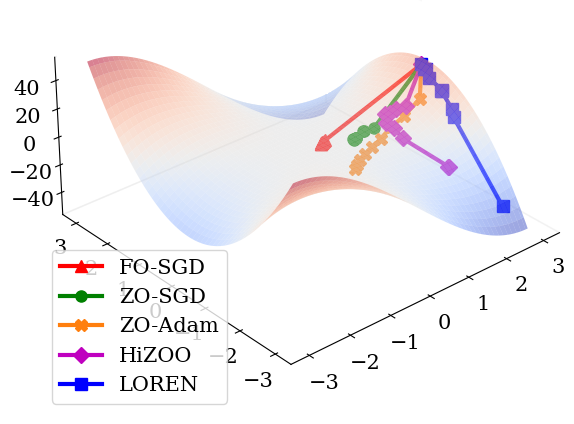}
        \subcaption{}
        \label{fig:optim_demo_b}
    \end{subfigure}
    \begin{subfigure}{0.36\textwidth}
        \centering
        \includegraphics[width=\textwidth]{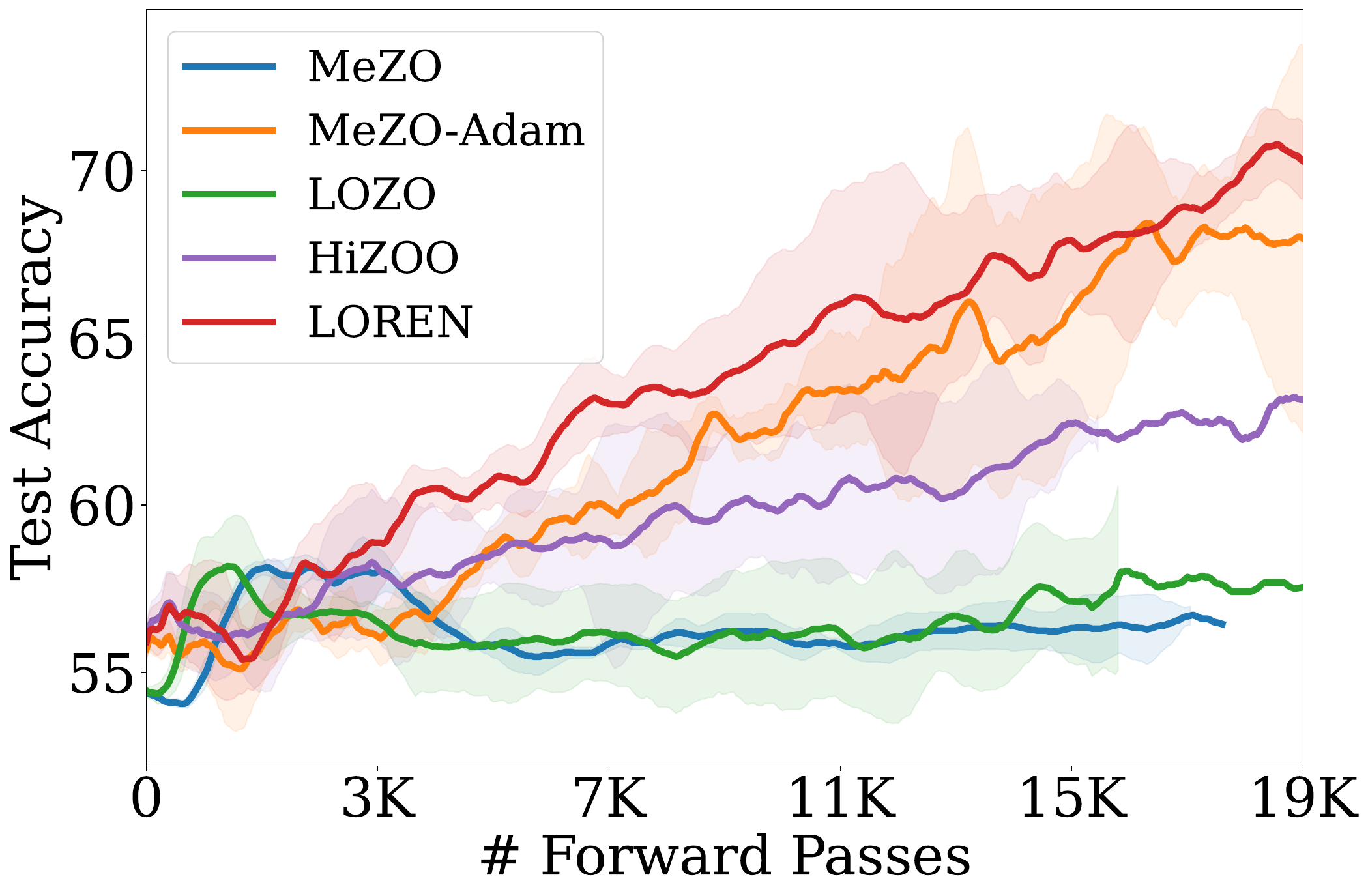}
        \subcaption{}
        \label{fig:optim_demo_c}
    \end{subfigure}
    \qquad
    \begin{subfigure}{0.352\textwidth}
        \centering
        \includegraphics[width=\textwidth]{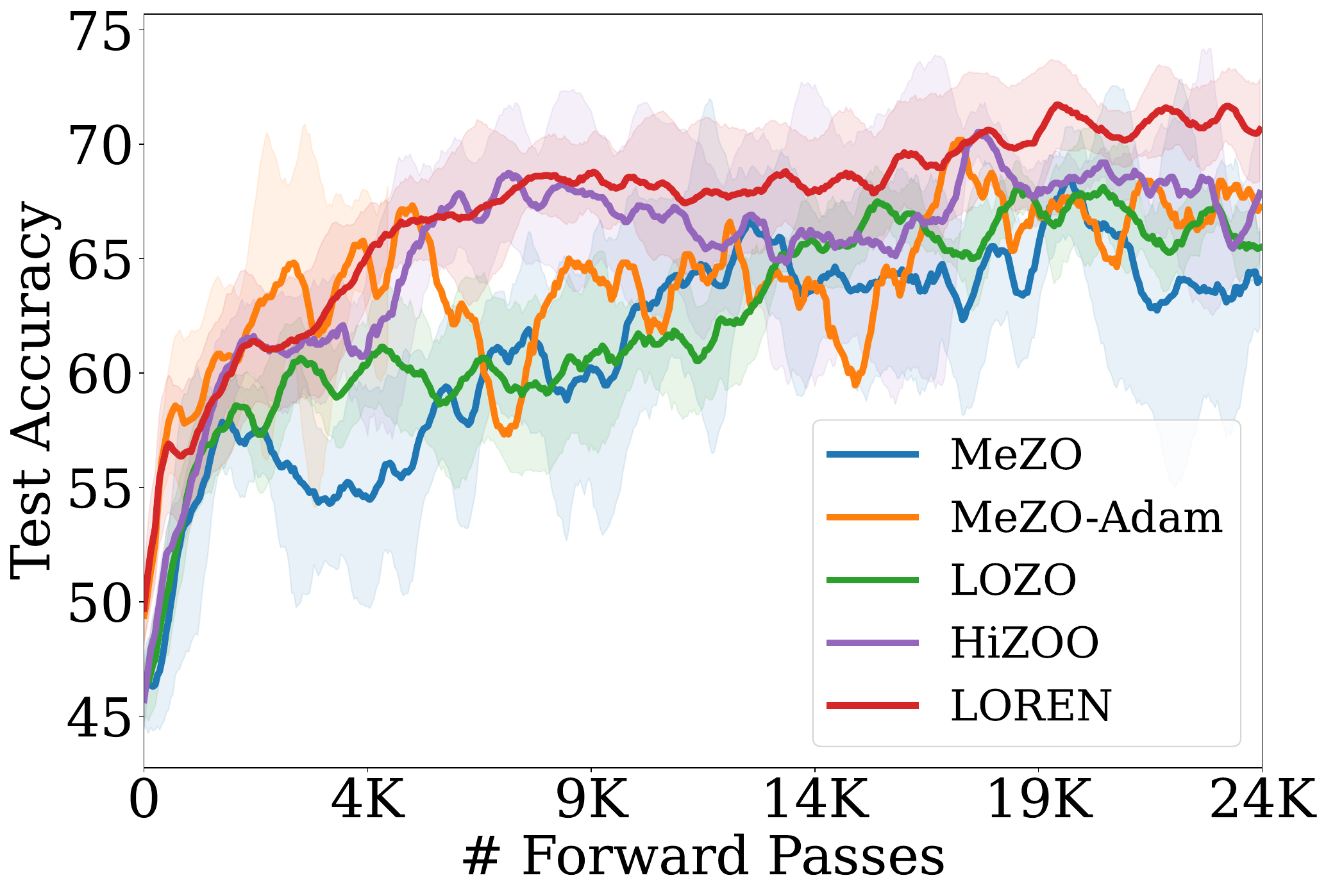}
        \subcaption{}
        \label{fig:optim_demo_d}
    \end{subfigure}
    \caption{\textbf{(a)} Mean squared errors of ZO gradient estimates, with and without RLOO, relative to the true gradient on the 1,000-dimensional Sphere, Rastrigin, and Rosenbrock functions. \textbf{(b)} Optimization trajectories of FO-SGD and ZO optimizers on the monkey saddle function, all initialized at (2.9, -0.01). Accuracy curves for \textbf{(c)} GPT-2-XL fine-tuned on QNLI and \textbf{(d)} OPT-13B fine-tuned on CB, using early stopping.}
    \label{fig:optim_demo}
\end{figure*}

In this paper, we propose \loren (\textbf{L}ow-rank
c\textbf{O}variance, \textbf{RE}INFORCE, and \textbf{N}atural
evolution strategies), a novel ZO optimization method designed to
overcome these challenges. \loren introduces three main innovations: 
\begin{enumerate}[label=(\roman*),leftmargin=*,noitemsep,topsep=0pt]
\item We reformulate the problem of gradient preconditioning in ZO
optimization as that of adaptively estimating a sampling
distribution from which random perturbations are drawn for
finite-difference gradient estimation. Existing ZO optimizers 
typically draw random perturbations from either isotropic Gaussian
or uniform distribution over unit sphere, which assume uniform
curvature in all directions and thus ignore the underlying geometry
of loss landscape. In contrast, \loren dynamically learns a
perturbation distribution that captures anisotropic local curvature
of loss function.
\item \loren leverages the framework of natural evolution strategy
(NES)~\cite{Rechenberg1973EvolutionsstrategieO,Wierstra2008NaturalES}
to accelerate the search for optimal parameters of the perturbation
distribution and models the inverse of Hessian using the Kronecker
factored rank-1 approximation to scale NES to LLM fine-tuning. The
Kronecker factorization approach allows \loren to approximate
curvature information with significantly reduced memory overhead,
making it suitable for LLM training.
\item Unlike traditional ZO optimizers that rely on 
finite-difference gradient estimators, \loren employs the REINFORCE
leave-one-out (RLOO)~\cite{Kool2019Buy4R} estimator to reduce the
variance and make effective use of multiple function evaluations.
\end{enumerate}
By combining these approaches, \loren produces search directions that are both well-conditioned and low-variance, all while preserving the memory-efficient nature of ZO methods. On Figure~\ref{fig:optim_demo_a}, we present the mean squared error (MSE) of gradient estimates for ZO-SGD, with and without RLOO, for three 1,000-dimensional test functions: Sphere, Rastrigin, and Rosenbrock. For each method, we generate 5,000 gradient estimates at a fixed point and compute their MSE relative to the true gradient. Both methods use four perturbations per gradient estimate for fairness. As shown, ZO-SGD with RLOO consistently achieves a lower MSE, demonstrating effective variance reduction. On Figure~\ref{fig:optim_demo_b}, we visualize the optimization trajectories of FO and ZO optimizers on the monkey saddle function. While ZO-SGD and ZO-Adam struggle due to noisy gradient estimates, HiZOO~\cite{zhao2025secondorder} shows moderate improvement by using a ZO-Hessian estimate.
Notably, \loren 
follows the most efficient path, 
escaping the saddle region by 
leveraging low-rank curvature and low-variance gradient estimates.

To validate the effectiveness of \loren, we evaluate its performance
in fine-tuning
both masked language models and autoregressive models on the GLUE~\cite{Wang2018GLUEAM} and SuperGLUE~\cite{Wang2019SuperGLUEAS} tasks. Figure~\ref{fig:optim_demo_c} and \ref{fig:optim_demo_d} present the test accuracy curves of state-of-the-art ZO optimizers, fine-tuning GPT-2-XL on QNLI and OPT-13B on CB, respectively. By incorporating curvature‐aware updates and variance reduction, \loren achieves the highest mean accuracy, demonstrating both superior performance and markedly more stable convergence.
The key contributions of our work can be summarized as follows: 
\begin{itemize}[leftmargin=*,noitemsep,topsep=0pt]
\item We introduce \loren, the first ZO optimizer that simultaneously adapts to curvature and applies variance reduction, enabling an efficient fine-tuning of LLMs. This combined approach delivers stable and scalable updates, even in high-dimensional and ill-conditioned settings.
\item We establish the link between ZO gradients and evolution strategies and directly estimate the preconditioned ZO gradients using the score function estimator without any additional forward passes. We construct a damped rank-1 covariance structure to preserve memory efficiency and ensure that the additional memory overhead to store curvature information remains negligible.
\item To the best of our knowledge, \loren is the first method to apply a block-diagonal approximation of the Hessian matrix in ZO optimization, capturing richer curvature information than a pure diagonal approximation.
\item We provide extensive experimental results on standard LLM benchmarks, comparing all leading ZO methods for fine-tuning. \loren consistently delivers higher test accuracy while maintaining lower memory footprint compared to other state-of-the-art preconditioned or adaptive ZO methods
, raising the bar for memory-efficient LLM fine-tuning.
\end{itemize}

\section{Related Work}
\label{sec:related}
Our work intersects ZO optimization, memory-efficient LLM fine-tuning, and curvature-aware variance reduction.

\textbf{ZO Optimization for LLMs }
ZO methods 
replace explicit gradients with function evaluations via finite-difference approximations
such as SPSA~\cite{Spall1992MultivariateSA}. This paradigm gained traction for LLM fine-tuning due to its potential for extreme memory efficiency compared to backpropagation. MeZO~\cite{malladi2023finetuning} pioneered this application, adapting ZO-SGD~\cite{Spall1992MultivariateSA} with an in-place implementation to match inference memory costs. While demonstrating feasibility and achieving strong results, MeZO can be sensitive to prompts and exhibits higher variance than FO methods.  
LOZO~\cite{chen2025enhancing} focused on aligning the ZO gradient estimator with the observed low-rank structure of LLM gradients, proposing a low-rank gradient estimator. 
\loren also uses a low-rank structure but 
applies it to the preconditioner (i.e., covariance matrix),
rather than directly estimating a low-rank gradient as in LOZO.

\textbf{Preconditioned ZO Methods }
To address slow convergence on ill-conditioned loss landscapes, curvature information has been incorporated.
HiZOO~\cite{zhao2025secondorder} 
incorporates second-order information by explicitly estimating
the diagonal entries of the Hessian
with an additional forward pass and uses it for preconditioning.
Other works have explored Hessian-aware ZO methods in different contexts rather than fine-tuning LLMs~\cite{Ye2018HessianAwareZO,Balasubramanian2018ZerothOrderNS}. 
\loren draws inspiration from the NES~\cite{Rechenberg1973EvolutionsstrategieO,Wierstra2008NaturalES}, adapting a low-rank Kronecker-factored approximation of the perturbation covariance matrix. 
This allows capturing curvature 
information
to guide the search direction efficiently without storing or estimating second-order elements directly.

\textbf{Variance Reduction in ZO Methods }
The high variance of ZO gradient estimators is a key
obstacle. 
MeZO-SVRG~\cite{gautam2024variancereduced} adapts SVRG~\cite{Johnson2013AcceleratingSG}, using periodic full-batch estimates to correct minibatch gradients, to improve 
stability and convergence over MeZO but at the cost of
increased memory, requiring storage for reference gradients and
parameters. \loren utilizes the RLOO~\cite{Kool2019Buy4R} method, a score function gradient estimator combined with a leave-one-out baseline, for gradient estimation. RLOO computes a baseline for each sample within a batch using the rewards (function values) of the other samples in the same batch. This provides effective variance reduction without needing full-batch computations like SVRG, thus preserving the minimal memory footprint of \loren.


\section{Preliminaries}
\label{sec:prelim}
\subsection{Notations}
Vectors are denoted by lowercase bold (e.g., $\vec{x}$), and matrices by uppercase bold (e.g., $\mat{X}$). 
%
We write $x_i$ to denote the $i$\th entry of vector $\vec{x}$.
$\norm{\vec{x}}$ represents the Euclidean norm unless otherwise stated. 
$\otimes$ represents the Kronecker product. For a matrix $\mat{X}\in\R^{m\times n}$, its vectorization is $\vect(\mat{X}) =
\begin{bmatrix}
\mat{X}_{*,1}^{\intercal} & \mat{X}_{*,2}^{\intercal} & \cdots & \mat{X}_{*,n}^{\intercal}
\end{bmatrix}^{\intercal}\,$, where $\mat{X}_{*,j}$ denotes the $j$\th
column of matrix $\mat{X}$. For two matrices $\mat{A}$ and $\mat{B}$, the symbol $:$ denotes their trace product, i.e., $\mat{A}:\mat{B} = \Trace(\mat{A}^\intercal \mat{B})$. 

\subsection{Zeroth-Order Gradient Estimates}
We consider the following stochastic optimization problem using the ZO oracle:
\[
  \underset{\vec{x\in \R^d}}{\argmin} f(\vec{x}):=\E_{\xi\sim\mathbb{P}}[\ell(\vec{x}; \xi)]\,,
\]
where $\vec{x}$ denotes the model parameters, $\xi$ denotes a random
data sample, and $f$ is the expected loss over the data
distribution. When FO gradients are inaccessible, a common
strategy is to estimate gradients using the finite-difference
method~\cite{Ghadimi2013StochasticFA,Nesterov2017RandomGM}. A widely
used technique is the Simultaneous Perturbation Stochastic
Approximation (SPSA), which estimates gradients using random
perturbations in all coordinates simultaneously. 
\begin{definition}[SPSA~\cite{Spall1992MultivariateSA}]
Let $f:\R^{d}\to \R$. For $\epsilon >0$, the SPSA gradient estimator is given by
\[
    \hat{\grad} f(\vec{x}) = \E_{\vec{u}\sim \mathcal{N}(\vec{0},
      \mat{I}_d)}\left[\frac{f(\vec{x} + \epsilon \vec{u}) -
        f(\vec{x}-\epsilon\vec{u})}{2\epsilon}\vec{u}\right]\,. 
\]
\end{definition}
$\hat{\grad} f(\vec{x})$ is closely related to the gradient of
Gaussian smoothed objective.
%

\begin{definition}[Generalized Gaussian smoothing] Let $f:\R^d \rightarrow \R$. The Gaussian smoothing of $f(\vec{x})$ is
  defined by
\[
  f_{\epsilon,\mat{\Sigma}}(\vec{x}) \overset{\Delta}{=} \E_{\vec{u}\sim
    \mathcal{N}(\vec{0}, \mat{\Sigma})}[f(\vec{x} + \epsilon\vec{u})]\,, 
\]
where $\epsilon > 0$ controls the smoothness.
\end{definition}
\begin{proposition} \label{prop:grad_gauss_rs}
  For $f:\R^{d}\to \R$, the gradient of Gaussian smoothed $f$ is given
  by
  \begin{equation} \label{eq:gauss_spsa} 
    \grad f_{\epsilon,\mat{\Sigma}}(\vec{x}) = \E_{\vec{u} \sim \mathcal{N}(\vec{0}, \mat{\Sigma})}
    \left[ \frac{f(\vec{x} +\epsilon \vec{u}) -
        f(\vec{x}-\epsilon\vec{u})}{2\epsilon} \mat{\Sigma}^{-1}\vec{u}
    \right]\,.
  \end{equation}  
\end{proposition}
Proposition~\ref{prop:grad_gauss_rs} shows that $\grad
f_{\epsilon,\mat{\Sigma}}(\vec{x}) = \hat{\grad} f(\vec{x})$ when
$\mat{\Sigma} = \mat{I}_d$.
%
As $\epsilon \to 0$, 
$\grad f_{\epsilon,\mat{\Sigma}}(\vec{x})$ approximates the true gradient:
\begin{align} \label{eq:spsa_unbiased}
    \lim_{\epsilon\to 0}\grad f_{\epsilon,\mat{\Sigma}}(\vec{x})
  &= \E_{\vec{u}\sim \mathcal{N}(\vec{0}, \mat{\Sigma})}[\langle\grad
  f(\vec{x}),\, \vec{u}\rangle \mat{\Sigma}^{-1}\vec{u}] \nonumber \\
  &= \E_{\vec{u}\sim \mathcal{N}(\vec{0},
    \mat{\Sigma})}[\mat{\Sigma}^{-1}\vec{u}\vec{u}^\intercal \grad f(\vec{x})] = \grad
  f(\vec{x})\,.  
\end{align}


\subsection{REINFORCE with Leave-One-Out Baseline}
To reduce variance, score function estimators are typically used with
a control variate $b$, referred to as a baseline that is independent
of $\vec{z}_k$:
\[
  \grad_{\vec{\theta}}J(\vec{\theta}) \approx
  \frac{1}{K}\sum_{k=1}^{K}(f(\vec{z}_k) - b)\grad_{\vec{\theta}}\log
  p(\vec{z}_k;\vec{\theta})\,,
\]
where $\vec{z}_k\sim p(\vec{z};\vec{\theta})$.
For $K\geq 2$, the RLOO estimator sets $b$
to the leave-one-out average of function values, leveraging 
multiple evaluations of $f$ to reduce variance.
\[
  \grad_{\vec{\theta}}J(\vec{\theta})\approx
  \frac{1}{K}\sum_{k=1}^{K}\left(f(\vec{z}_k) -
    \frac{\sum_{j\neq
      k}f(\vec{z}_j)}{K-1}\right)\grad_{\vec{\theta}}\log
  p(\vec{z}_k;\vec{\theta})\,. 
\]
The estimator can be equivalently expressed as $\frac{1}{K-1}\sum_{k=1}^{K}\left(f(\vec{z}_k) -
    \frac{1}{K}\sum_{j=1}^{k}f(\vec{z}_j)\right)\grad_{\vec{\theta}}\log
  p(\vec{z}_k;\vec{\theta})$.



\section{Methods}
\label{sec:method}
This section presents a detailed derivation of \loren. 

\subsection{Preconditioning via Evolution Strategies} \label{subsec:second_order_info}
Consider the following preconditioned gradient update:
\[
    \vec{x} \gets \vec{x} -\eta \Tilde{\mat{H}}^{-1}\grad f(\vec{x})\,,
\]
where $\Tilde{\mat{H}}$ is a symmetric positive definite matrix
approximating the curvature information. 
By replacing the perturbation vector $\vec{u}$ in~\eqref{eq:spsa_unbiased}
with the scaled Gaussian $\Tilde{\mat{H}}^{-1/2}\vec{u}$, we get
\begin{align} \label{eq:precondgd_cov} 
  \Tilde{\mat{H}}^{-1}\grad f(\vec{x})
  &= \E_{\vec{u}\sim \mathcal{N}(\vec{0}, \mat{I}_d)}[%
    \Tilde{\mat{H}}^{-1/2}\vec{u}\vec{u}^\intercal
    \Tilde{\mat{H}}^{-1/2}\grad f(\vec{x})] \nonumber \\
    &= \E_{\vec{u}\sim \mathcal{N}(\vec{0}, \mat{I}_d)}[%
    \langle\grad f(\vec{x}),\,
    \Tilde{\mat{H}}^{-1/2}\vec{u}\rangle
    \Tilde{\mat{H}}^{-1/2}\vec{u}] \nonumber  \\
  &= \E_{\Tilde{\vec{u}}\sim \mathcal{N}(\vec{0},
    \Tilde{\mat{H}}^{-1})}\left[\langle\grad f(\vec{x}),\,
    \Tilde{\vec{u}}\rangle \Tilde{\vec{u}}\right] \nonumber \\
    &\approx
    \E_{\Tilde{\vec{u}}\sim \mathcal{N}(\vec{0},
    \Tilde{\mat{H}}^{-1})}\left[ \frac{f(\vec{x} + \epsilon
    \Tilde{\vec{u}}) - f(\vec{x})}{\epsilon}\Tilde{\vec{u}}
    \right]\,. 
\end{align}
Equation~\eqref{eq:precondgd_cov} demonstrates that preconditioning
the gradient in ZO optimization is equivalent to drawing the
perturbation vector $\tilde{\vec{u}}$ from an anisotropic Gaussian
distribution whose covariance matrix
$\mat{\Sigma}$ equals the inverse of curvature
matrix $\Tilde{\mat{H}}$, i.e., 
$\mat{\Sigma}=\Tilde{\mat{H}}^{-1} = 
\epsilon^{2}\overline{\mat{\Sigma}}$.
We estimate the gradient in~\eqref{eq:precondgd_cov} using the
framework of evolution strategies (ES)~\cite{Rechenberg1973EvolutionsstrategieO}:
\begin{align*}
    \underset{\vec{\theta}}{\argmin}J(\vec{\theta})
  &:= \E_{\vec{z}\sim p(\vec{z}; \vec{\theta})}[f(\vec{z})] \\
  &= \E_{\vec{u}\sim\mathcal{N}(\vec{0}, \mat{I}_d)}\left[f(\vec{x}
    + \epsilon\overline{\mat{\Sigma}}^{1/2}\vec{u})\right]\,,
\end{align*}
where $\epsilon$ is a smoothing parameter and $p(\vec{z}; \vec{\theta}) = \mathcal{N}(\vec{x},
\epsilon^{2}\overline{\mat{\Sigma}})$ is the search
  distribution whose mean is the current 
solution (i.e., model parameters) $\vec{x}$ and the covariance matrix $\epsilon^{2}\overline{\mat{\Sigma}}$
models the inverse of the curvature matrix.
The gradient of $J(\vec{\theta})$ can be calculated using the score 
function estimator (also known as the REINFORCE
estimator~\cite{Williams1992SimpleSG}), given by
\begin{align}
  \grad_{\vec{\theta}}J(\vec{\theta})
  &= \E_{\vec{z}\sim p(\vec{z}; \vec{\theta})}\left[f(\vec{z})\nabla_{\vec{\theta}}
    \log{p(\vec{z}; \vec{\theta})}\right]\,. \label{eq:scorefunest} 
\end{align}

\subsection{Low-Rank Structured Covariance Matrices}
Consider a network layer with parameters $\vec{x}=\vect(\mat{X})\in
\R^{mn}$, where $\mat{X}\in \R^{m\times n}$. While the ES framework
allows capturing local curvature information, it requires
maintaining and updating the covariance matrix $\mat{\Sigma}\in \R^{mn\times
  mn}$, which can incur prohibitive memory cost, particularly for
LLMs. 
Second-order optimizers
such as Shampoo~\cite{Gupta2018ShampooPS} and
KFAC~\cite{martens2015optimizing} exploit Kronecker-factored
curvature approximations 
to efficiently estimate the curvature matrix using
significantly smaller memory than
storing the full matrix.
Recent studies~\cite{sankar2021Deeper,yang2022Does, seung2025mac} have shown that
the Hessian and Fisher Information matrix (FIM) of deep neural
networks exhibit inherent low-rank structure. 
Motivated by these, we propose to 
estimate the curvature matrix
$\tilde{\mat{H}} = \grad_{\vec{x}}^2f(\vec{x})\in \R^{mn\times mn}$ by 
\begin{equation} \label{eq:lowrank_approx}
  \tilde{\mat{H}} = \mat{I}_{m} \otimes (\rho\mat{I}_n + \vec{a}\vec{a}^\intercal) \,,
\end{equation}
where 
$\rho>0$ is a damping factor and $\vec{a}\in \R^{n}$ is a learnable
vector that parameterizes the curvature matrix.
The damped rank-1 block-diagonal approximation admits a closed-form solution for
both the inverse $\tilde{\mat{H}}^{-1}$ and the inverse square root
$\tilde{\mat{H}}^{-1/2}$, enabling efficient implementation.
As shown in~\ref{subsec:second_order_info}, we
leverage the curvature information directly by setting $\mat{\Sigma} = 
\tilde{\mat{H}}^{-1}$ and draw the perturbation $\tilde{\vec{u}} =
\mat{\Sigma}^{1/2}\vec{u}$, where $\vec{u}\sim\mathcal{N}(\vec{0}, \mat{I}_{mn})$,
%
%
\begin{align}
  &\mat{\Sigma}
  =  \mat{I}_{m} \otimes (\rho\mat{I}_n +
    \vec{a}\vec{a} ^\intercal)^{-1} = \mat{I}_{m}
    \otimes \frac{1}{\rho}\left(\mat{I}_n -
    \frac{\vec{a}\vec{a}^\intercal}{\rho +
    \|\vec{a}\|^2}\right)\,,  \label{eq:cov}\\ 
  &\mat{\Sigma}^{1/2}
  = \mat{I}_m \otimes \frac{1}{\sqrt{\rho}}\left(\mat{I}_n
    -\frac{\sqrt{\rho}+\sqrt{\rho +
    \|\vec{a}\|^2}}{\|\vec{a}\|^2\sqrt{\rho +
    \|\vec{a}\|^2}}\vec{a}\vec{a}^\intercal\right)\,, \text{ and } 
    \label{eq:sqrt_cov}\\ 
  &\mat{d \Sigma}
  = \mat{I}_{m} \otimes
    \left(-\frac{(\vec{d}\vec{a})\vec{a}^\intercal +
    \vec{a}(\vec{d}\vec{a}^\intercal)}{\rho(\rho +
    \|\vec{a}\|^2)} + \frac{2\vec{a}^\intercal(\vec{d}\vec{a})\vec{a}\vec{a}^\intercal}{\rho(\rho
    + \|\vec{a}\|^2)^2}\right)\,.  \label{eq:d_cov} 
\end{align}

\subsection{Search Distribution Parameter Updates}
Let $p(\vec{z};\vec{\theta})$ denote the search distribution, modeled as
a multivariate Gaussian 
%
$\mathcal{N}(\vec{x},\mat{\Sigma})$
in \loren, where 
$\vec{\theta} = [\vec{x}^\intercal,\vect(\mat{\Sigma})^\intercal]^\intercal$
and $\mat{\Sigma}$ is defined as in~\eqref{eq:cov}.
%
Our goal is to update the parameters $\vec{\theta}$ of search
distribution $p(\vec{z};\vec{\theta})$ such that the expected loss
$J(\vec{\theta}) = \E_{\vec{z}\sim
  p(\vec{z};\vec{\theta})}[f(\vec{z})]$ of the underlying model is
minimized. 
Let $\mathcal{L}=\log p(\vec{z};\vec{\theta})$.
The differential of $\mathcal{L}$ is given by (see Appendix~\ref{apdx:derive} for a
complete derivation):
\begin{align*}
  \dd{\mathcal{L}}
  &=
    \frac{1}{2}\mat{\Sigma}^{-1}(\mat{Z}-\mat{\Sigma})\mat{\Sigma}^{-1}
    : \dd{\mat{\Sigma}} + \mat{\Sigma}^{-1}(\vec{z}-\vec{x}) :
    \dd{\vec{x}}\,, 
\end{align*}
where $\mat{Z}=(\vec{z}-\vec{x})(\vec{z}-\vec{x})^\intercal$ and
$\dd{\mat{Z}} = -\dd{\vec{x}}(\vec{z}-\vec{x})^\intercal -
(\vec{z}-\vec{x})(\dd{\vec{x}})^\intercal$. Applying~\eqref{eq:cov}
and \eqref{eq:d_cov} gives
\begin{align}
  \dd{\mathcal{L}}
  &= (\mat{Z}-\mat{\Sigma}) : (\mat{I}_m \otimes
    -(\dd{\vec{a}})\vec{a}^\intercal))  + (\mat{Z}-\mat{\Sigma}) : \left(\mat{I}_m \otimes
    \left(-\frac{\vec{a}^\intercal\dd{\vec{a}}}{\rho}\vec{a}\vec{a}^\intercal\right)\right)
    \nonumber \\ 
  &\qquad +\frac{1}{2}(\mat{Z}-\mat{\Sigma}) : \left(\mat{I}_m
    \otimes \left(\frac{2\vec{a}^\intercal\dd{\vec{a}}}{\rho}\vec{a}\vec{a}^\intercal\right)\right) + \mat{\Sigma}^{-1}(\vec{z}-\vec{x}) : \dd{\vec{x}} \nonumber \\
  &= -\sum_{i=1}^{m}(\mat{Z}-\mat{\Sigma})_{ii}\vec{a} : \dd{\vec{a}} +
  \mat{\Sigma}^{-1}(\vec{z}-\vec{x}) : \dd{\vec{x}}\,, \label{eq:d_L} 
\end{align}
where the last equality is due to Proposition~\ref{prop:trace_ikronb}.
\begin{proposition} \label{prop:trace_ikronb}
Let $\mat{A}\in \R^{mn \times mn}$ be a symmetric matrix, and
$\mat{B}\in \R^{n\times n}$. Then we have $\Trace(\mat{A}(\mat{I}_m
\otimes \mat{B})) = \Trace((\sum_{i=1}^{m}\mat{A}_{ii})\mat{B})$,
where $\mat{A}_{ij}$ denotes the submatrix located at the $(i,j)$-th block
position when $\mat{A}$ is viewed as an $m\times m$ block matrix with
each block of size $n\times n$.
\end{proposition}
From~\eqref{eq:d_L}, we obtain
\begin{align*}
    \grad_{\vec{x}}\log p(\vec{z};\vec{\theta}) &=
    \mat{\Sigma}^{-1}(\vec{z}-\vec{x})\,, \\
    \grad_{\vec{a}}\log p(\vec{z};\vec{\theta}) &=
    \sum_{i=1}^{m}(\mat{Z}-\mat{\Sigma})_{ii}\vec{a}\,. 
\end{align*}
Applying the reparameterization trick $\vec{z} = \vec{x} +
\mat{\Sigma}^{1/2} \vec{u}$, where $\vec{u}\sim \mathcal{N}(\vec{0},
\mat{I}_{mn})$, and~\eqref{eq:sqrt_cov} yields 
\begin{align}
  \grad_{\vec{a}} \log p(\vec{z};\vec{\theta}) &= \sum_{i=1}^{m}(\mat{\Sigma}^{1/2}(\mat{I}_{mn} -
    \vec{u}\vec{u}^\intercal)\mat{\Sigma}^{1/2})_{ii}\vec{a} \nonumber
  \\ 
  &=\sum_{i=1}^{m} \frac{\big(\mat{M}_i\vec{a} - \kappa(\vec{a}^\intercal\mat{M}_i\vec{a}) \vec{a}\big)}{\sqrt{\rho}\sqrt{\rho +
      \|\vec{a}\|^2}}\,, \label{eq:grad_a} 
\end{align}
where $\kappa = \left(\sqrt{\rho} + \sqrt{\rho +
      \|\vec{a}\|^2}\right)/ \left(\|\vec{a}\|^2 \sqrt{\rho +
      \|\vec{a}\|^2}\right)$,  $\mat{M}_i = \vec{u}_i\vec{u}_i^\intercal -
      \mat{I}_n$, and $\vec{u}_i\in \R^{n}$ denotes the $i$\th subvector of
      $\vec{u}\in \R^{mn}$ from index $(i-1)n+1$ to $in$, for $i=1,
      \ldots, m$.
The FIM $\mat{F}_{\vec{x}}$ has a simple form and is given by
\begin{align*}
    \mat{F}_{\vec{x}}
  &= \E_{\vec{z}\sim p(\vec{z};\vec{\theta})}[\grad_{\vec{x}}\log p(\vec{z};\vec{\theta})
  \grad_{\vec{x}}\log p(\vec{z};\vec{\theta})^\intercal]\\
  &= \E_{\vec{u}\sim \mathcal{N}(\vec{0}, \mat{I}_{mn})} [%
  \mat{\Sigma}^{-1/2}\vec{u}\vec{u}^\intercal \mat{\Sigma}^{-1/2}] \\ 
      &= \mat{\Sigma}^{-1}\,.
\end{align*}
Thus, 
the natural gradient w.r.t. $\vec{x}$ is given by 
\begin{equation}
  \mat{F}_{\vec{x}}^{-1}\grad_{\vec{x}}\log p(\vec{z};\vec{\theta})
  = \mat{\Sigma}\mat{\Sigma}^{-1}(\vec{z}-\vec{x}) = \mat{\Sigma}^{1/2}\vec{u}\,. \label{eq:grad_mu}
\end{equation}
%
We now derive the score function 
estimates of ZO gradient
$\hat{\grad}_{\vec{x}} f(\vec{x})$ and $\hat{\grad}_{\vec{a}}
f(\vec{x})$. 
%
%
Using
the fact that the gradient of Gaussian smoothed $f(\vec{x})$
corresponds to the SPSA estimator as given in~\eqref{eq:gauss_spsa},
we obtain 
\begin{align*}
  \hat{\grad}_{\vec{x}} f(\vec{x})
  &= \grad_{\vec{x}}\E_{\vec{u}\sim\mathcal{N}(\vec{0},
    \mat{I}_d)}\left[f(\vec{x} +
    \epsilon\overline{\mat{\Sigma}}^{1/2}\vec{u})\right]  \\
  &= \E_{\vec{u}\sim \mathcal{N}(\vec{0}, \mat{I}_{mn})}\left[f(\vec{x}
    +
    \epsilon\overline{\mat{\Sigma}}^{1/2}\vec{u})\epsilon^{-1}
    \overline{\mat{\Sigma}}^{-1/2}\vec{u}
    \right]\,. 
\end{align*}
From~\eqref{eq:grad_mu}, the natural gradient
$\mat{F}_{\vec{x}}^{-1}\hat{\grad}_{\vec{x}} f(\vec{x})$ is given by
\[
  \mat{F}_{\vec{x}}^{-1}\hat{\grad}_{\vec{x}} f(\vec{x})
  = \E_{\vec{u}\sim \mathcal{N}(\vec{0}, \mat{I}_{mn})}\left[f(\vec{x} +
    \epsilon\overline{\mat{\Sigma}}^{1/2}\vec{u})\epsilon^{-1}
    \overline{\mat{\Sigma}}^{1/2}\vec{u}
  \right]\,,
\]
where, from~\eqref{eq:sqrt_cov},
\begin{align}
  \overline{\mat{\Sigma}}^{1/2}\vec{u}
  &= \frac{1}{\epsilon}\left\{\mat{I}_m \otimes \frac{1}{\sqrt{\rho}}\left(\mat{I}_n
    -\kappa
    \vec{a}\vec{a}^\intercal\right)\right\}\vec{u} \nonumber \\
  &= \frac{1}{\epsilon}\sum_{i=1}^{m}\frac{1}{\sqrt{\rho}}\left(\mat{I}_n
    -\kappa\vec{a}\vec{a}^\intercal\right)\vec{u}_i\,. \label{eq:sqrt_cov_u} 
\end{align}
Similarly, we have
\[
  \hat{\grad}_{\vec{a}} f(\vec{x}) = \E_{\vec{u}\sim
    \mathcal{N}(\vec{0}, \mat{I}_{mn})}\left[f(\vec{x} +
    \epsilon\overline{\mat{\Sigma}}^{1/2}\vec{u})\grad_{\vec{a}} \log{p(\vec{z};
      \vec{\theta})}\right]\,, 
\]
where $\grad_{\vec{a}} \log{p(\vec{z}; \vec{\theta})}$ is given
by~\eqref{eq:grad_a}.

\begin{algorithm}[tb]
  \KwIn{Dataset $S=\{\xi_1, \ldots, \xi_n\}$, Initialization
    $\vec{x}_0\in \R^{d}$, $\vec{a}_0 \sim \mathcal{N}(0, \mat{I}_d)$, number of iterations $T$, learning rates
    $\{\eta, \nu\}$, smoothing $\epsilon$, damping $\rho$, number of forward passes $K$}
  \DontPrintSemicolon
  \For{$t = 0$ \KwTo $T-1$}{
    Sample mini‐batch $\mathcal{B}_t$\;
    \For{$k = 1$ \KwTo $K$}{
      Sample $\vec{u}_k \sim\mathcal{N}\left(\vec{0}, \mat{I}_d\right)$ \label{alg:sample_u}\;
      $f^{k} \gets f(\vec{x}_t +
      \epsilon\overline{\mat{\Sigma}}^{1/2}\vec{u}_{k}; \mathcal{B}_t)$ \label{alg:compute_f}\;
    }
    $\vec{g}(\vec{x}_t) \gets \frac{1}{\epsilon(K-1)}\sum_{k=1}^{K}\left(f^k - \frac{1}{K}\sum_{j=1}^{K}f^j\right)\overline{\mat{\Sigma}}^{1/2}\vec{u}_k$ \label{alg:rloo_grad_x}\;
    $\vec{g}(\vec{a}_t) \gets \frac{1}{K-1}\sum_{k=1}^{K}\left(f^k - \frac{1}{K}\sum_{j=1}^{K}f^j\right)\grad_{\vec{a}}\log p(\vec{z}_k;\vec{\theta})$, \quad where $\vec{z}_k = \vec{x}_k + \epsilon\overline{\mat{\Sigma}}^{1/2}\vec{u}_k$ \label{alg:rloo_grad_a}\;
    $\vec{x}_{t+1} \gets \vec{x}_t - \eta \vec{g}(\vec{x}_t)$ \label{alg:update_x}\;
    $\vec{a}_{t+1} \gets \vec{a}_t - \nu \vec{g}(\vec{a}_t)$ \label{alg:update_a}\;
  }
  \Return $\vec{x}_{T}$\;
  \caption{LOREN}
  \label{alg:loren}
\end{algorithm}

\subsection{Algorithm}
The key steps of \loren are summarized in the pseudocode presented in
Algorithm~\ref{alg:loren}.
The algorithm begins by sampling $K$ perturbation vectors $\vec{u}_k \sim
\mathcal{N}(\vec{0}, \mat{I}_d)$, where the dimensionality $d=mn$ for layers with
matrix-valued parameters and $d=m$ for layers with vector-valued
parameters. The loss function $f$ is then evaluated 
using the scaled perturbations
$\epsilon\overline{\mat{\Sigma}}^{1/2}\vec{u}_k$ as shown in~\eqref{eq:sqrt_cov_u}
(Lines~\ref{alg:sample_u}--\ref{alg:compute_f}).
In Lines~\ref{alg:rloo_grad_x}--\ref{alg:rloo_grad_a}, \loren applies
the RLOO estimator to compute the variance reduced gradients w.r.t. both the mean 
$\vec{x}$ and covariance parameters $\vec{a}$.
Using the gradient estimates,
Lines~\ref{alg:update_x}--\ref{alg:update_a} 
simultaneously update
$\vec{x}$ and 
$\vec{a}$.

\begin{table}[tb]
  \caption{Additional memory requirement compared to MeZO.}
  \label{tab:complexity}
  \centering
  \footnotesize
  \begin{tabular}{lccccc}
    \toprule
    Method & MeZO-Adam & MeZO-SVRG & LOZO & HiZOO & \loren \\
    \midrule
    Cost & $\mathcal{O}(mn)$ & $\mathcal{O}(mn)$ & $\mathcal{O}(nr)$ & $\mathcal{O}(mn)$ & $\mathcal{O}(n)$  \\
    \bottomrule
  \end{tabular}
\end{table}

\textbf{Memory Complexity }
Table~\ref{tab:complexity} shows the additional memory overhead of 
MeZO variants relative to MeZO. 
MeZO-Adam, MeZO-SVRG, and HiZOO each require 
$\mathcal{O}(mn)$ extra space, while LOZO incurs 
$\mathcal{O}(nr)$ overhead to store low-rank gradient components, where $r$ is the rank. In contrast, \loren reduces the memory cost to 
$\mathcal{O}(n)$ by maintaining only low-rank covariance components. 
This memory efficiency enables the use of heavyball momentum~\cite{polyak1964}, from which \loren benefits through acceleration. We refer to this momentum variant simply as \loren in Section~\ref{sec:experiment}.

\subsection{Convergence Analysis}
The following theorem shows that \loren can converge to a stationary
point at a rate of $\mathcal{O}(1/\sqrt{T})$, where $T$ is the number of iterations. 
\begin{theorem} \label{thm:convergence}
Assuming the $L$-smoothness of the objective function $f$ and bounded
variance of gradient estimates (see Assumption~\ref{assm:smoothness}
and~\ref{assm:bounded_variance} in Appendix~\ref{sec:supp_convergence}
for details), the sequence of parameter vectors $\{\vec{x}_t\}$
generated by Algorithm~\ref{alg:loren} with $\eta =
\frac{1}{8\sqrt{T}L(\max_t \Trace(\mat{\Sigma}_t) + 2\rho^{-1})}$
satisfies
\begin{align*}
   \min_{t=1:T}\E\left[\norm{\grad{f}(\vec{x}_t;\xi_t)}^{2}\right] &\leq
  \frac{32L(\max_t \Trace(\mat{\Sigma}_t) +
    2\rho^{-1})\left(f(\vec{x}_1; \xi_1) -
      f(\vec{x}_{*};\xi_{*})\right)}{\sqrt{T}\alpha_{\min}} \\
      &\qquad + \frac{\sigma^{2}}{\sqrt{T}\alpha_{\min}} + \mathcal{O}(\epsilon^{2})\,, 
\end{align*}
where $\alpha_{\min} = (\rho + \max_t\norm{\vec{a}_t}^{2})^{-1}$ is the
smallest eigenvalue of $\mat{\Sigma}_t$.
\end{theorem}
Given $\mat{\Sigma}_t = \mat{I}_m \otimes \frac{1}{\rho}\left(\mat{I}_n - \frac{\vec{a}_t \vec{a}_t^\top}{\rho + \|\vec{a}_t\|^2}\right)$, we have
\begin{align*}
\mathrm{tr}(\mat{\Sigma}_t) &= \mathrm{tr}(\mat{I}_m) \cdot \mathrm{tr}\!\left(
  \frac{1}{\rho} \left( \mat{I}_n - \frac{\vec{a}_t \vec{a}_t^\top}{\rho + \|\vec{a}_t\|^2}\right)\right) \\
&=\frac{m}{\rho}\left(n - \frac{\|\vec{a}_t\|^2}{\rho + \|\vec{a}_t\|^2}\right)\,.
\end{align*}
Since $0 \le \frac{\|\vec{a}_t\|^2}{\rho + \|\vec{a}_t\|^2} < 1$,
$\frac{m(n-1)}{\rho}
\;\le\;
\mathrm{tr}(\mat{\Sigma}_t)$ holds. Substituting into $\eta$ gives
\[
\eta
=
\frac{1}{8 \sqrt{T} \, L \, \bigl(\max_t \mathrm{tr}(\mat{\Sigma}_t) + 2 \rho^{-1}\bigr)}
\;\le\;
\frac{\rho}{8 \sqrt{T} \, L \, \bigl(m(n-1) + 2\bigr)}.
\]

\begin{table}[tb]
    \centering
    \caption{Experimental results on DistilBERT and RoBERTa. Reported metrics include best accuracy (\%) with standard deviation over 5 runs and the averaged accuracy across 4 benchmark tasks from GLUE.}
    \footnotesize
    \begin{tabular}{lccccc}
        \toprule
        & \multicolumn{5}{c}{\textbf{DistilBERT (66M) --- FP32}} \\
        \cmidrule(lr){2-6}
        Task & MNLI & QNLI & SST-2 & CoLA & Avg  \\ 
        \midrule
        MeZO & 39.9\scriptsize{$\pm$0.2} & 48.5\scriptsize{$\pm$0.6} & 62.1\scriptsize{$\pm$0.2} & 67.0\scriptsize{$\pm$0.4} & 54.4  \\
        MeZO-Adam & 41.2\scriptsize{$\pm$1.5} & 71.1\scriptsize{$\pm$2.2} & 78.4\scriptsize{$\pm$1.7} & 67.8\scriptsize{$\pm$1.9} & 64.6  \\
        MeZO-SVRG & 42.0\scriptsize{$\pm$1.5} & 64.0\scriptsize{$\pm$2.4} & 73.6\scriptsize{$\pm$2.7} & 66.6\scriptsize{$\pm$0.9} & 61.6 \\
        LOZO & 40.2\scriptsize{$\pm$0.3} & 53.0\scriptsize{$\pm$1.4} & 61.0\scriptsize{$\pm$1.5} & 67.0\scriptsize{$\pm$0.6} & 55.3  \\
        HiZOO & 40.0\scriptsize{$\pm$0.2} & 64.5\scriptsize{$\pm$8.0} & 78.7\scriptsize{$\pm$0.9} & 67.0\scriptsize{$\pm$1.0} & 62.6  \\
        \midrule
        \loren & 39.8\scriptsize{$\pm$0.0} & 73.0\scriptsize{$\pm$2.0} & 81.7\scriptsize{$\pm$1.0} & 67.2\scriptsize{$\pm$0.8} & \textbf{65.4} \\
        \midrule
        & \multicolumn{5}{c}{\textbf{RoBERTa‐large (355M) --- FP32}} \\
        \cmidrule(lr){2-6}
        Task & MNLI & QNLI & SST-2 & CoLA & Avg \\ 
        \midrule
        MeZO & 39.8\scriptsize{$\pm$0.3} & 71.6\scriptsize{$\pm$1.5} & 54.8\scriptsize{$\pm$0.6} & 67.2\scriptsize{$\pm$0.3} & 58.4  \\
        MeZO‐Adam & 51.6\scriptsize{$\pm$1.5} & 80.1\scriptsize{$\pm$1.8} & 84.8\scriptsize{$\pm$3.7} & 77.9\scriptsize{$\pm$1.5} & \textbf{73.6}  \\
        MeZO‐SVRG & 39.8\scriptsize{$\pm$0.3} & 59.1\scriptsize{$\pm$6.7} & 55.2\scriptsize{$\pm$0.5} & 67.3\scriptsize{$\pm$0.5} & 54.6 \\
        LOZO & 41.3\scriptsize{$\pm$2.2} & 70.8\scriptsize{$\pm$1.7} & 54.6\scriptsize{$\pm$0.5} & 69.5\scriptsize{$\pm$1.8} & 59.1 \\
        HiZOO & 43.1\scriptsize{$\pm$0.9} & 70.2\scriptsize{$\pm$2.5} & 73.2\scriptsize{$\pm$5.5} & 70.2\scriptsize{$\pm$1.4} & 64.2\\
        \midrule
        \loren & 44.3\scriptsize{$\pm$1.4} & 76.3\scriptsize{$\pm$1.5} & 86.1\scriptsize{$\pm$3.1} & 73.8\scriptsize{$\pm$0.4} & 70.1 \\
        \bottomrule
    \end{tabular}
    \label{tab:masked_models}
\end{table}


\section{Experiments}
\label{sec:experiment}
We evaluate \loren on masked and autoregressive language models using both GLUE and SuperGLUE benchmarks, comparing it with MeZO, MeZO-Adam, MeZO-SVRG, LOZO, and HiZOO. 
For optimal performance, \loren employs six forward evaluations per iteration with RLOO variance reduction.
The same evaluation budget is applied to all baselines to ensure a fair and consistent comparison. While other ZO optimizers typically use only two or three forward passes per step, this default setting leads to degraded performance compared to using six passes. Results under their default settings are provided in Appendix~\ref{apdx:2fwd_pass}.
We conduct full-parameter fine-tuning of LLMs without using prompts, following the more challenging setting from \cite{gautam2024variancereduced}, in contrast to the prompt fine-tuning setup used in \cite{malladi2023finetuning, chen2025enhancing, zhao2025secondorder}, in order to better highlight the performance gap between ZO optimizers.
Early stopping~\cite{Prechelt1996EarlySW} was employed to prevent over-iteration, as ZO optimizers tend to yield diminishing performance gains once convergence is reached. All experiments were conducted on a single NVIDIA H100 or A100 GPU, with details provided in Appendix~\ref{apdx:details}. An ablation study exploring the impact of \loren’s key hyperparameters is presented in Appendix~\ref{apdx:ablation}.

\begin{table}[tb]
    \centering
    \caption{Experimental results on GPT-2, OPT, and LLaMA-3. Reported metrics include best accuracy (\%) with standard deviation over 5 runs and the average accuracy across 4 benchmark tasks from GLUE and SuperGLUE.}
    \footnotesize
    \begin{tabular}{lccccc}
        \toprule
        & \multicolumn{5}{c}{\textbf{GPT-2-XL (1.5B) --- FP32}} \\
        \cmidrule(lr){2-6}
        Task & MNLI & QNLI & SST-2 & CoLA & Avg \\ 
        \midrule
        MeZO & 39.1\scriptsize{$\pm$1.1} & 58.8\scriptsize{$\pm$0.2} & 73.8\scriptsize{$\pm$0.8} & 65.4\scriptsize{$\pm$0.2} & 59.3 \\
        MeZO-Adam & 50.9\scriptsize{$\pm$1.3} & 72.3\scriptsize{$\pm$4.3} & 91.2\scriptsize{$\pm$0.6} & 71.6\scriptsize{$\pm$0.8} & 71.5 \\
        MeZO-SVRG & 49.4\scriptsize{$\pm$1.0} & 65.2\scriptsize{$\pm$1.0} & 84.0\scriptsize{$\pm$1.6} & 65.8\scriptsize{$\pm$0.2} & 66.1 \\
        LOZO & 42.1\scriptsize{$\pm$0.7} & 60.0\scriptsize{$\pm$1.3} & 79.4\scriptsize{$\pm$1.0} & 65.6\scriptsize{$\pm$0.3} & 61.8 \\ 
        HiZOO & 48.6\scriptsize{$\pm$0.2} & 66.3\scriptsize{$\pm$3.6} & 89.6\scriptsize{$\pm$0.2} & 71.5\scriptsize{$\pm$0.8} & 69.0 \\ 
        \midrule
        \loren & 51.2\scriptsize{$\pm$0.3} & 74.6\scriptsize{$\pm$1.2} & 89.8\scriptsize{$\pm$0.8} & 72.0\scriptsize{$\pm$0.7} & \textbf{71.9} \\
        \midrule
        & \multicolumn{5}{c}{\textbf{OPT-2.7B --- FP32}} \\
        \cmidrule(lr){2-6}
        Task & MNLI & QNLI & SST-2 & CoLA & Avg \\ 
        \midrule
        MeZO & 50.2\scriptsize{$\pm$3.3} & 75.2\scriptsize{$\pm$1.0} & 88.9\scriptsize{$\pm$0.2} & 68.4\scriptsize{$\pm$2.3} & 70.7 \\
        MeZO-Adam & 45.7\scriptsize{$\pm$1.2} & 72.0\scriptsize{$\pm$3.1} & 86.7\scriptsize{$\pm$2.7} & 68.2\scriptsize{$\pm$1.0} & 68.2 \\
        MeZO-SVRG & 41.2\scriptsize{$\pm$0.2} & 59.4\scriptsize{$\pm$0.8} & 63.5\scriptsize{$\pm$1.8} & 66.2\scriptsize{$\pm$0.2} & 57.6  \\
        LOZO & 42.7\scriptsize{$\pm$1.0} & 59.8\scriptsize{$\pm$0.8} & 66.4\scriptsize{$\pm$3.1} & 66.9\scriptsize{$\pm$0.2} & 59.0  \\
        HiZOO & 50.4\scriptsize{$\pm$1.2} & 75.4\scriptsize{$\pm$0.8} & 88.1\scriptsize{$\pm$0.6} & 66.0\scriptsize{$\pm$0.0} & 70.0  \\
        \midrule
        \loren & 55.1\scriptsize{$\pm$0.7} & 73.8\scriptsize{$\pm$1.6} & 89.5\scriptsize{$\pm$0.8} & 68.0\scriptsize{$\pm$0.0}  & \textbf{71.6} \\
        \midrule
        & \multicolumn{5}{c}{\textbf{LLaMA-3-8B --- BF16}} \\
        \cmidrule(lr){2-6}
        Task & RTE & BoolQ & WiC & CB & Avg \\ 
        \midrule
        MeZO & 59.6\scriptsize{$\pm$1.5} & 64.2\scriptsize{$\pm$0.8} & 59.2\scriptsize{$\pm$1.9} & 71.4\scriptsize{$\pm$1.5} & 63.6  \\
        MeZO-Adam & 58.6\scriptsize{$\pm$1.0} & 63.7\scriptsize{$\pm$0.0} & 57.7\scriptsize{$\pm$1.0} & 68.5\scriptsize{$\pm$0.8} & 62.1 \\
        MeZO-SVRG & 58.2\scriptsize{$\pm$1.4} & 63.8\scriptsize{$\pm$0.2} & 58.0\scriptsize{$\pm$1.1} & 65.7\scriptsize{$\pm$2.9} & 61.4  \\
        LOZO & 57.0\scriptsize{$\pm$1.2} & 64.8\scriptsize{$\pm$0.6} & 57.2\scriptsize{$\pm$2.1} & 71.4\scriptsize{$\pm$4.4} & 62.6  \\
        HiZOO & 57.8\scriptsize{$\pm$0.6} & 63.7\scriptsize{$\pm$0.0} & 59.2\scriptsize{$\pm$1.5} & 66.7\scriptsize{$\pm$3.7} & 61.8 \\
        \midrule
        \loren & 58.6\scriptsize{$\pm$1.0} & 64.8\scriptsize{$\pm$1.2} & 59.2\scriptsize{$\pm$1.2} & 72.2\scriptsize{$\pm$2.7} & \textbf{63.7}\\
        \midrule
        & \multicolumn{5}{c}{\textbf{OPT-13B --- BF16}} \\
        \cmidrule(lr){2-6}
        Task & RTE & BoolQ & WiC & CB & Avg  \\ 
        \midrule
        MeZO & 59.2\scriptsize{$\pm$0.5} & 64.3\scriptsize{$\pm$1.0} & 57.2\scriptsize{$\pm$0.4} & 73.1\scriptsize{$\pm$1.6} & 63.4 \\
        MeZO-Adam & 60.8\scriptsize{$\pm$1.4} & 63.4\scriptsize{$\pm$1.3} & 55.7\scriptsize{$\pm$1.2} & 70.2\scriptsize{$\pm$1.7} & 62.5 \\
        MeZO-SVRG & 59.1\scriptsize{$\pm$0.9} & 65.2\scriptsize{$\pm$0.7} & 57.3\scriptsize{$\pm$0.9} & 69.3\scriptsize{$\pm$2.8} & 62.7  \\
        LOZO & 58.5\scriptsize{$\pm$0.7} & 63.2\scriptsize{$\pm$0.2} & 57.6\scriptsize{$\pm$1.8} & 68.5\scriptsize{$\pm$3.0} & 62.0  \\
        HiZOO & 59.8\scriptsize{$\pm$1.3} & 63.4\scriptsize{$\pm$0.4} & 57.4\scriptsize{$\pm$1.0} & 72.6\scriptsize{$\pm$1.7} & 63.3  \\
        \midrule
        \loren & 60.0\scriptsize{$\pm$0.7} & 63.2\scriptsize{$\pm$0.6} & 57.2\scriptsize{$\pm$0.8} & 73.7\scriptsize{$\pm$1.3} &\textbf{63.5}  \\
        \bottomrule
    \end{tabular}
    \label{tab:ar_models}
\end{table}

\begin{figure*}[tb]
    \centering
    \begin{subfigure}{0.245\textwidth}
        \centering
        \includegraphics[width=\textwidth]{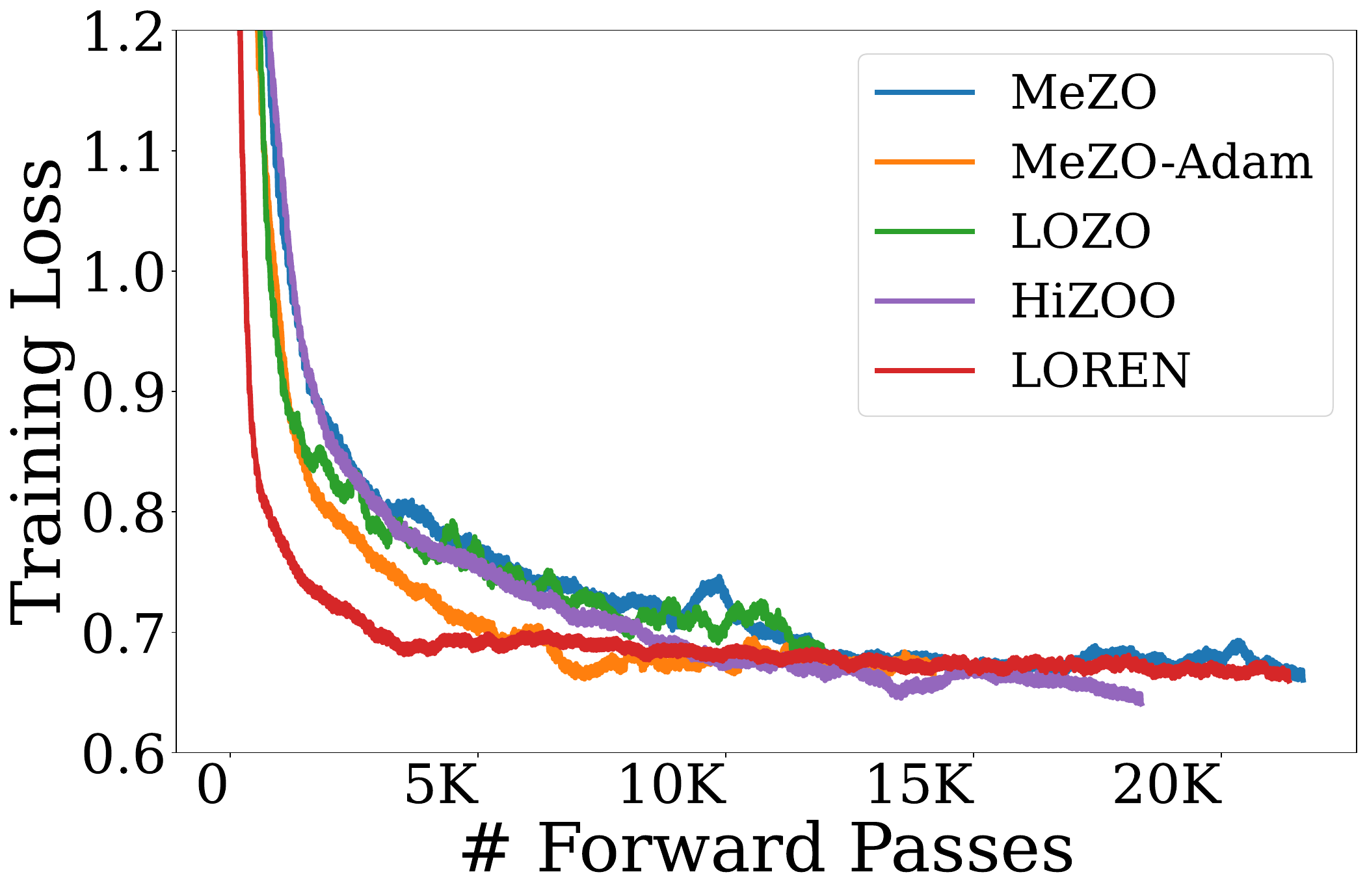}
        \subcaption{RTE}
        \label{fig:train_loss_rte}
    \end{subfigure}%
    \begin{subfigure}{0.24\textwidth}
        \centering
        \includegraphics[width=\textwidth]{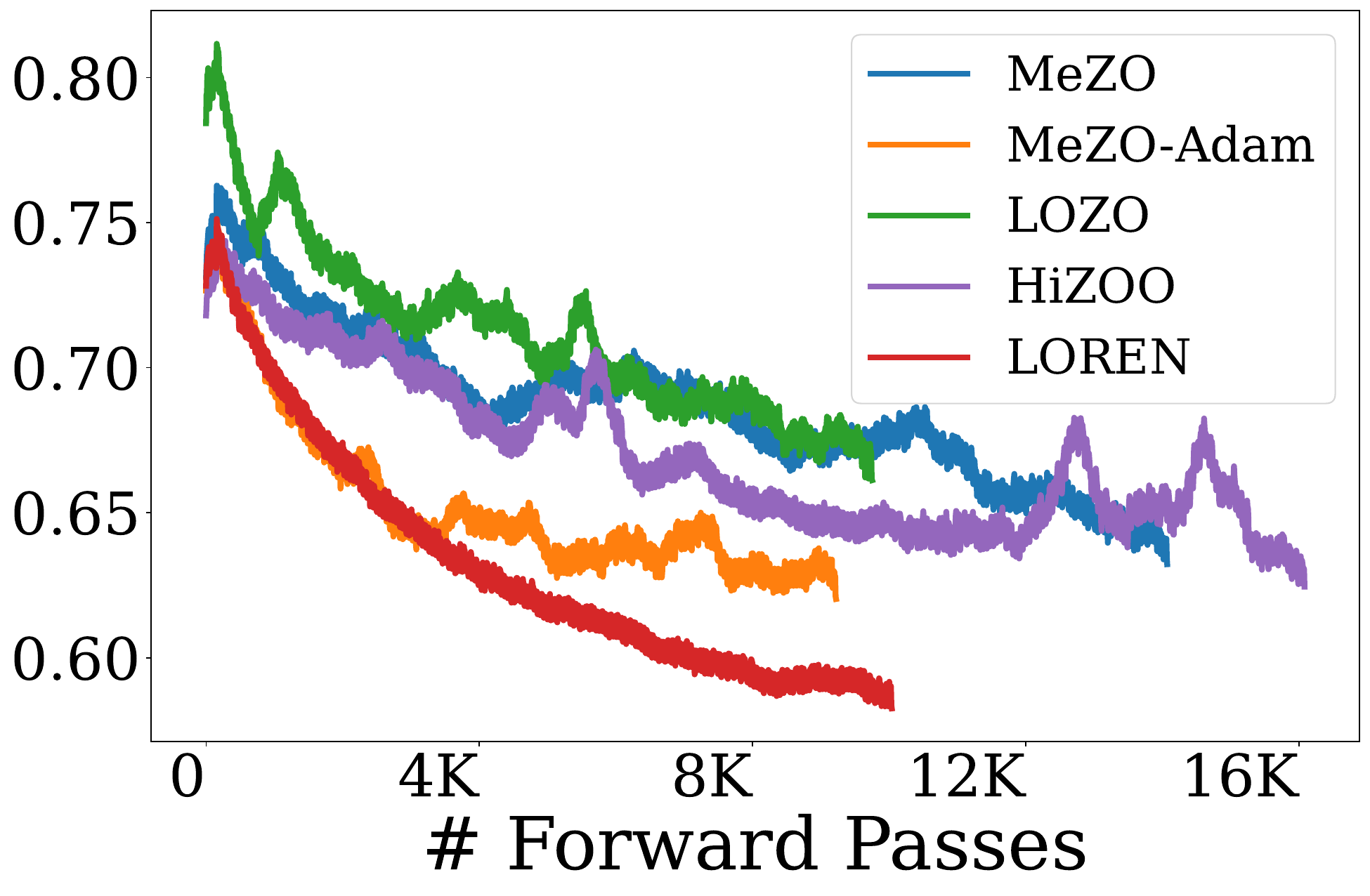}
        \subcaption{BoolQ}
        \label{fig:train_loss_boolq}
    \end{subfigure}%
    \begin{subfigure}{0.245\textwidth}
        \centering
        \includegraphics[width=\textwidth]{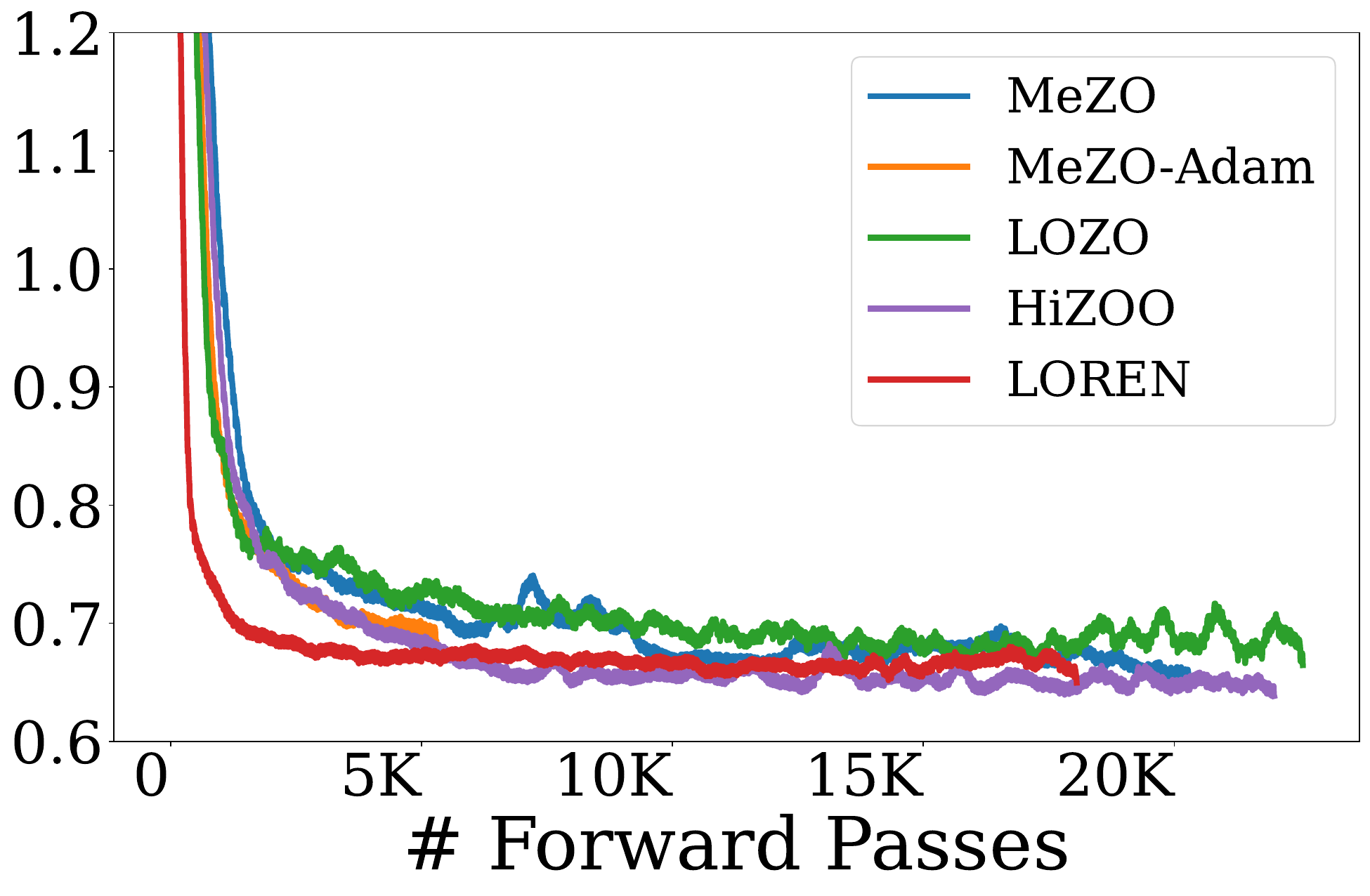}
        \subcaption{WiC}
        \label{fig:train_loss_wic}
    \end{subfigure}%
    \begin{subfigure}{0.24\textwidth}
        \centering
        \includegraphics[width=\textwidth]{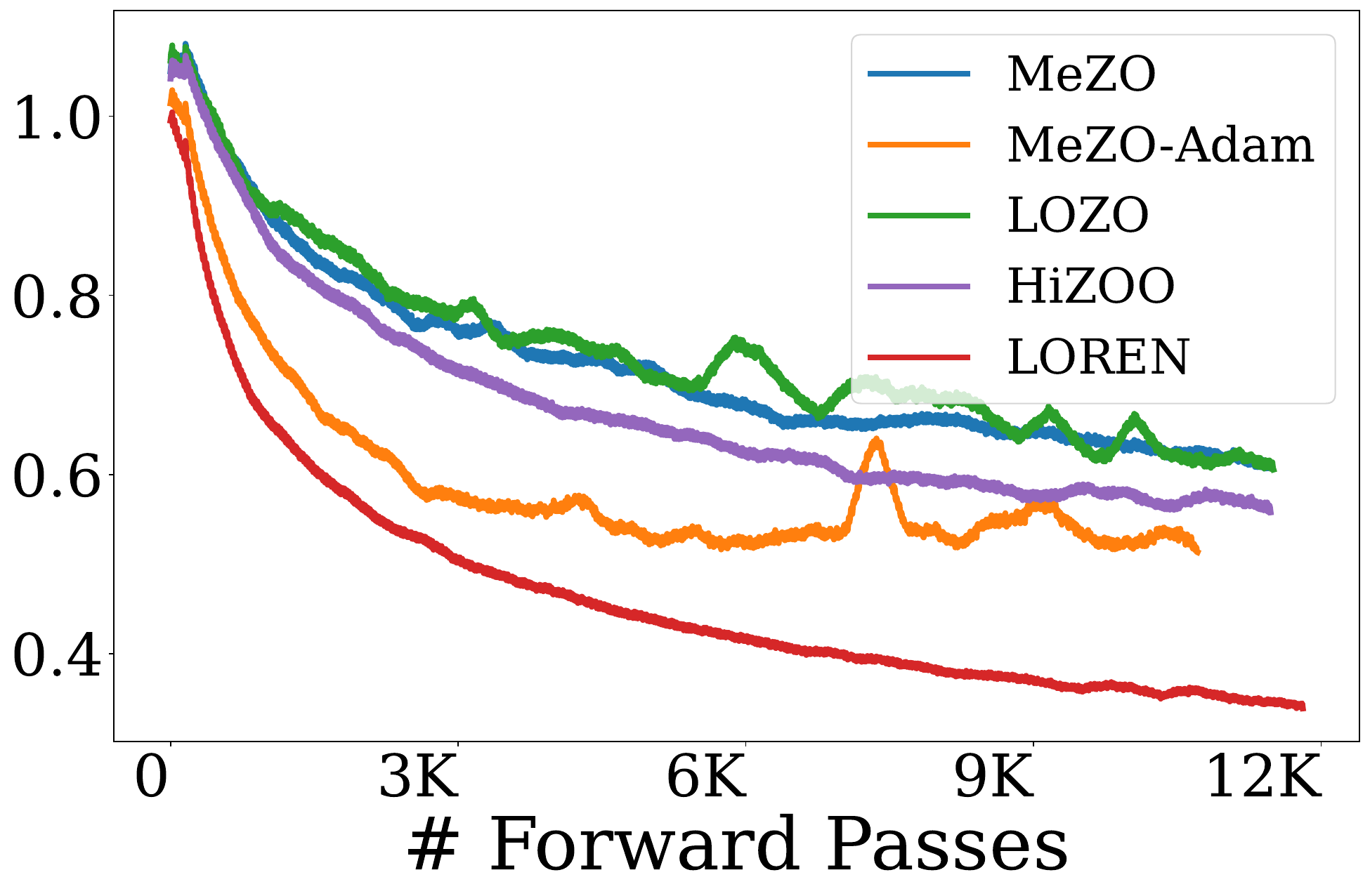}
        \subcaption{CB}
        \label{fig:train_loss_cb}
    \end{subfigure}
    \caption{Training loss curves for different ZO optimizers when fine-tuning OPT-13B on SuperGLUE tasks.}
    \label{fig:train_loss_superglue}
\end{figure*}

\begin{table*}[tb]
    \centering
    \caption{Peak GPU memory consumption (GB) and relative usage (MeZO = 1.00). LLaMA-3-8B and OPT-13B were trained using half-precision (BF16).}
    \scriptsize
    \begin{tabular}{lcccccc}
        \toprule
        Models & DistilBERT & RoBERTa-large & GPT-2-XL & OPT-2.7B & LLaMA-3-8B & OPT-13B \\
        \midrule
        MeZO & 0.85 & 2.14 & 16.9 & 17.4 & 18.4 & 32.9 \\
        MeZO-Adam & 1.36 (1.60$\times$) & 4.83 (2.26$\times$) & 28.8 (1.70$\times$) & 37.3 (2.14$\times$) & 46.2 (2.51$\times$) & 76.0 (2.31$\times$) \\
        MeZO-SVRG & 1.18 (1.39$\times$) & 4.25 (1.99$\times$) & 32.3 (1.91$\times$) & 32.6 (1.87$\times$) & 44.3 (2.41$\times$) & 74.7 (2.27$\times$) \\
        LOZO & 0.76 (0.89$\times$) & 2.07 (0.97$\times$) & 16.8 (0.99$\times$) & 16.8 (0.97$\times$) & 17.4 (0.95$\times$) & 32.8 (1.00$\times$) \\
        HiZOO & 1.43 (1.68$\times$) & 4.49 (2.10$\times$) & 24.3 (1.44$\times$) & 29.3 (1.69$\times$) & 36.5 (1.98$\times$) & 59.6 (1.81$\times$) \\
        \midrule
        \loren & 1.15 (1.35$\times$) & 3.73 (1.74$\times$) & 23.1 (1.37$\times$) & 27.6 (1.59$\times$) & 33.6 (1.83$\times$) & 57.5 (1.75$\times$) \\
        \bottomrule
    \end{tabular}
    \label{tab:memory_usage}
\end{table*}

\subsection{Masked Language Models}
\textbf{\loren consistently improves performance across masked language models. } As shown in Table~\ref{tab:masked_models}, \loren achieves the highest average accuracy on DistilBERT, outperforming both MeZO and LOZO by more than 10 percentage points (\%p). On RoBERTa-large, \loren ranks second overall, closely trailing MeZO-Adam, while surpassing MeZO, MeZO-SVRG, and LOZO by over 10\%p, and HiZOO by 6\%p.

\subsection{Autoregressive Language Models} \label{sec:ar_lm}
\textbf{\loren delivers the best overall accuracy across all model architectures. }
As shown in Table~\ref{tab:ar_models}, on GPT-2-XL \loren slightly improves over MeZO-Adam and outperforms MeZO-SVRG and HiZOO by more than 5\%p. On OPT-2.7B, it leads MeZO-Adam by over 3\%p and MeZO-SVRG and LOZO by over 10\%p. Across SuperGLUE tasks with LLaMA-3 and OPT-13B, \loren again achieves the highest average accuracy, consistently surpassing every baseline. We exclude MeZO-SVRG from our SuperGLUE fine-tuning because its performance gains are limited and its implementation runs considerably slower on large models.

\textbf{\loren achieves the fastest loss minimization. } The training loss curves in 
Figure~\ref{fig:train_loss_superglue} 
confirms \loren’s markedly faster convergence.
For OPT-13B fine-tuned on SuperGLUE, \loren demonstrates the most rapid loss reduction among ZO methods, ultimately reaching substantially lower final losses on BoolQ and CB compared to all baselines. Similar results highlighting \loren's fastest loss minimization on GLUE tasks are presented in Appendix~\ref{apdx:glue_train_loss}.

\subsection{Memory and Training Efficiency}
\textbf{\loren still maintains the affordable memory usage.}  
As shown in Table~\ref{tab:memory_usage}, even after integrating variance reduction, curvature-aware updates, and momentum, \loren still requires less memory than MeZO-Adam, MeZO-SVRG, and HiZOO. Although it does not match the minimal usage of MeZO or LOZO, \loren achieves a favorable trade-off. Across six architectures, \loren's peak memory consumption ranges from 1.35$\times$ to 1.83$\times$ that of the MeZO baseline, compared to 1.68$\times$ -- 2.10$\times$ for HiZOO, 1.60$\times$ -- 2.51$\times$ for MeZO-Adam, and 1.39$\times$ -- 2.41$\times$ for MeZO-SVRG. These results indicate that the additional memory overhead introduced by \loren remains relatively modest in comparison to other MeZO variants. 

\textbf{\loren exhibits the highest query efficiency.}  
Table~\ref{tab:train_efficiency} reports the number of forward passes and wall-clock time needed to reach a target accuracy when fine-tuning GPT-2-XL on SST-2 and LLaMA-3-8B on CB. We set the targets based on the lowest accuracy achieved by any ZO optimizer in each setting. In both benchmarks, \loren requires the fewest forward-pass queries to reach the targets. In terms of wall-clock time, \loren consistently ranks as the second fastest method, closely matching MeZO-Adam on GPT-2-XL fine-tuned for SST-2, and trailing only LOZO on LLaMA-3-8B fine-tuned for CB, while maintaining a clear runtime advantage over all other ZO methods.

\begin{table}[tb]
    \centering
    \caption{Number of forward passes and wall-clock time (hours) required to reach 70\% accuracy on SST-2 with GPT-2-XL and 65\% accuracy on CB with LLaMA-3-8B. All values are reported as mean $\pm$ standard deviation over 5 independent runs.}
    \footnotesize
    \begin{tabular}{lcccc}
    \toprule
    Models & \multicolumn{2}{c}{GPT-2-XL} & \multicolumn{2}{c}{LLaMA-3-8B}\\
    \cmidrule(lr){2-3} 
    \cmidrule(lr){4-5}
    & \# Queries & Time & \# Queries & Time \\ 
    \midrule
    MeZO & 16,752$\pm$816 & 3.13$\pm$0.2 & 5,736$\pm$1,476 & 0.80$\pm$0.2\\
    MeZO-Adam & 1,632$\pm$144 & \textbf{0.32$\pm$0.0} & 6,894$\pm$894 & 0.96$\pm$0.1\\
    MeZO-SVRG & 4,248$\pm$648 & 1.51$\pm$0.2 & 3,216$\pm$1,491 & 0.53$\pm$0.2 \\
    LOZO & 10,686$\pm$1,368 & 0.66$\pm$0.1 & 3,198$\pm$888 & \textbf{0.15$\pm$0.0}  \\
    HiZOO & 2,232$\pm$216 & 0.98$\pm$0.1 & 4,356$\pm$1,932 & 0.76$\pm$0.3\\
    \midrule
    \loren & \textbf{1,320$\pm$72} & 0.33$\pm$0.0 & \textbf{1,512$\pm$486} & 0.47$\pm$0.1\\
    \bottomrule
    \end{tabular}
    \label{tab:train_efficiency}
\end{table}

\section{Conclusions}
\label{sec:conclusion}
In this work, we proposed \loren, the first ZO preconditioned method
specifically designed to address the heterogeneous curvature problem
in LLM fine-tuning by learning an anisotropic random perturbation
distribution.
By combining the Kronecker-factored low-rank approximation of
curvature information with variance-reduced RLOO gradient
estimates, \loren effectively adapts to the geometry of complex loss
landscapes in a memory-efficient manner.
Empirical evaluations on LLM fine-tuning tasks
demonstrate that \loren consistently achieves higher accuracy and faster convergence across
various models
while maintaining a lower memory consumption compared to
state-of-the-art ZO methods.


\bibliographystyle{plain}
\bibliography{main}
\appendix
\onecolumn
\section*{Appendix}
\addcontentsline{toc}{section}{Appendix}
\setcounter{secnumdepth}{1}

\section{Derivation of Gradients} \label{apdx:derive}
Recall that for two matrices $\mat{A}$ and $\mat{B}$, the symbol $:$ denotes their trace product such that $\mat{A}:\mat{B = \Trace(\mat{A}^\intercal \mat{B})}$. We have
\begin{align*}
    \mat{A}:\mat{B} &= \mat{B}:\mat{A} = \mat{A}^\intercal : \mat{B}^\intercal \\
    \mat{A}:\mat{B}\mat{C} &= \mat{B}^\intercal\mat{A}:\mat{C} = \mat{A}\mat{C}^\intercal : \mat{B}
\end{align*}

\subsection{Gradients in Multivariate Gaussian Distributions}
Let $p(\vec{z};\vec{\theta})$ denote the multivariate Gaussian distribution $\mathcal{N}(\vec{x},\mat{\Sigma})$, where $\vec{\theta} = [\vec{x}^\intercal,\vect(\mat{\Sigma})^\intercal]^\intercal$, whose probability density function is given by
\begin{align*}
    p(\vec{z};\vec{\theta}) &= (2\pi)^{-d/2}\det(\mat{\Sigma})^{-1/2}\exp\left(-\frac{1}{2}(\vec{z}-\vec{x})^\intercal\mat{\Sigma}^{-1}(\vec{z}-\vec{x})\right)\\
    &=\Det(\mat{L})^{-1}\kappa \left(\mat{L}^{-1}(\vec{z}-\vec{x})\right)\,,
\end{align*}
where $\kappa(\vec{z}) = (2\pi)^{-d/2}\exp\left(-\frac{1}{2}\|\vec{z}\|^2\right)$ and $\mat{\Sigma} = \mat{L}\mat{L}^\intercal$.
\begin{align*}
    \mathcal{L}:=\log{p(\vec{z};\vec{\theta})} &=-\frac{d}{2}\log{(2\pi)} -\frac{1}{2}\log{\Det(\mat{\Sigma})} - \frac{1}{2}(\vec{z}-\vec{x})^\intercal \mat{\Sigma}^{-1}(\vec{z}-\vec{x}) \\
    &= \text{constant} - \frac{1}{2}\log{\Det(\mat{\Sigma})} - \frac{1}{2}\Trace\left(\mat{\Sigma}^{-1}(\vec{z}-\vec{x})(\vec{z}-\vec{x})^\intercal\right)\,.
\end{align*}
To compute the gradient, we can differentiate $\mathcal{L}$.
\begin{align}
    \dd\mathcal{L} &= -\frac{1}{2}\Trace\left(\mat{\Sigma}^{-1}\dd\mat{\Sigma}\right) - \frac{1}{2}\Trace\left(\dd\mat{\Sigma}^{-1}\mat{Z}\right) - \frac{1}{2}\Trace\left(\mat{\Sigma}^{-1}\dd\mat{Z}\right) \nonumber \\
    &= -\frac{1}{2}\Trace\left(\mat{\Sigma}^{-1}\dd\mat{\Sigma}\right) + \frac{1}{2}\Trace\left(\mat{\Sigma}^{-1}\dd\mat{\Sigma}\mat{\Sigma}^{-1}\mat{Z}\right) + \frac{1}{2} \Trace\left(\mat{\Sigma}^{-1}\left((\vec{z}-\vec{x})(\dd\vec{x})^\intercal + (\dd\vec{x})(\vec{z}-\vec{x})^\intercal \right)\right) \nonumber\\
    &= \frac{1}{2}\Trace\left(\dd\mat{\Sigma}\left(\mat{\Sigma}^{-1}\mat{Z}\mat{\Sigma}^{-1} - \mat{\Sigma}^{-1}\right)\right) + \Trace\left((\dd\vec{x})^\intercal \mat{\Sigma}^{-1}(\vec{z}-\vec{x})\right) \nonumber\\
    &= \frac{1}{2}\Trace\left(\dd\mat{\Sigma}\mat{\Sigma}^{-1}\left(\mat{Z} - \mat{\Sigma}\right)\mat{\Sigma}^{-1}\right) + (\dd\vec{x})^\intercal \mat{\Sigma}^{-1}(\vec{z}-\vec{x})\,, \label{eq:differntial}
\end{align}
where $\mat{Z} = (\vec{z}- \vec{x})(\vec{z}- \vec{x})^\intercal$ and $\dd\mat{Z} = -\dd\vec{x}(\vec{z}- \vec{x})^\intercal - (\vec{z}- \vec{x})(\dd\vec{x})^\intercal$.
Then we have
\[
    \grad_{\vec{\theta}}\log{p(\vec{z};\vec{\theta})} = \begin{bmatrix}
         \left(\mat{\Sigma}^{-1}(\vec{z}-\vec{x})\right)^\intercal & \frac{1}{2}\vect(\mat{\Sigma}^{-1}\mat{Z}\mat{\Sigma}^{-1} - \mat{\Sigma}^{-1})^\intercal
     \end{bmatrix}^\intercal\,.
\]

\subsection{Gradients for a Low-Rank Covariance Structure}
As in \eqref{eq:lowrank_approx}, define $\mat{\Sigma}^{-1} = \tilde{\mat{H}}$ as
\begin{align*}
  \mat{\Sigma}^{-1} &= \mat{I}_{m} \otimes (\rho\mat{I}_n + \vec{a}\vec{a}^\intercal) \\
  &=\left(\mat{I}_m \otimes \sqrt{\rho}\left(\mat{I}_n + \frac{\vec{a}\vec{a}^\intercal}{\sqrt{\rho}(\rho + \|\vec{a}\|)}\right)\right) \left(\mat{I}_m \otimes \sqrt{\rho}\left(\mat{I}_n + \frac{\vec{a}\vec{a}^\intercal}{\sqrt{\rho}(\rho + \|\vec{a}\|)}\right)\right)^\intercal \,,
\end{align*}
where $\vec{a}\in \R^{n}$. Its inverse is obtained by applying the Sherman-Morrison formula:
\begin{align*}
    \mat{\Sigma} &= \mat{I}_m \otimes (\rho\mat{I}_n + \vec{a}\vec{a}^\intercal)^{-1} \\
    &= \mat{I}_m \otimes \frac{1}{\rho}\left(\mat{I}_n - \frac{\vec{a\vec{a}^\intercal}}{\rho + \|\vec{a}\|}\right) = \mat{I}_m \otimes \mat{\Gamma} \\
    &= \left\{\mat{I}_m \otimes \frac{1}{\sqrt{\rho}}\left(\mat{I}_n - \frac{\sqrt{\rho} + \sqrt{\rho + \|\vec{a}\|^2}}{\|\vec{a}\|^2\sqrt{\rho + \|\vec{a}\|^2}}\vec{a}\vec{a}^\intercal\right)\right\}\left\{\mat{I}_m \otimes \frac{1}{\sqrt{\rho}}\left(\mat{I}_n - \frac{\sqrt{\rho} + \sqrt{\rho + \|\vec{a}\|^2}}{\|\vec{a}\|^2\sqrt{\rho + \|\vec{a}\|^2}}\vec{a}\vec{a}^\intercal\right)\right\}^\intercal \,, \\
    \dd\mat{\Sigma}&= \mat{I}_m \otimes \dd\mat{\Gamma} \\
    &= \mat{I}_m \otimes \left(-\frac{(\dd\vec{a})\vec{a}^\intercal + \vec{a}(\dd\vec{a}^\intercal)}{\rho(\rho + \|\vec{a}\|^2)}+\frac{2\vec{a}^\intercal (\dd\vec{a}) \vec{a}\vec{a}^\intercal}{\rho(\rho + \|\vec{a}\|^2)^2}\right)\,.
\end{align*}
Using the above, we derive the differential of $\mathcal{L}$ given in \eqref{eq:differntial}.
\begin{align*}
    \dd\mathcal{L} &= \frac{1}{2}\Trace\left(\dd\mat{\Sigma}\mat{\Sigma}^{-1}\left(\mat{Z} - \mat{\Sigma}\right)\mat{\Sigma}^{-1}\right) + (\dd\vec{x})^\intercal \mat{\Sigma}^{-1}(\vec{z}-\vec{x}) \\
    &= \frac{1}{2}\mat{\Sigma}^{-1}\left(\mat{Z} - \mat{\Sigma}\right)\mat{\Sigma}^{-1} : \dd\mat{\Sigma} +  \mat{\Sigma}^{-1}(\vec{z}-\vec{x}) : \dd\vec{x} \\
    &= \frac{1}{2} \mat{\Sigma}^{-1}\left(\mat{Z} - \mat{\Sigma}\right)\mat{\Sigma}^{-1} : \left\{\mat{I}_m \otimes \left(-\frac{(\dd\vec{a})\vec{a}^\intercal + \vec{a}(\dd\vec{a}^\intercal)}{\rho(\rho + \|\vec{a}\|^2)}+\frac{2\vec{a}^\intercal (\dd\vec{a}) \vec{a}\vec{a}^\intercal}{\rho(\rho + \|\vec{a}\|^2)^2}\right)\right\} +  \mat{\Sigma}^{-1}(\vec{z}-\vec{x}) : \dd\vec{x} \\
    &= \mat{\Sigma}^{-1}\left(\mat{Z} - \mat{\Sigma}\right)\mat{\Sigma}^{-1} : \left\{\mat{I}_m \otimes \left(-\frac{(\dd\vec{a})\vec{a}^\intercal}{\rho(\rho + \|\vec{a}\|^2)}\right)\right\} + \mat{\Sigma}^{-1}\left(\mat{Z} - \mat{\Sigma}\right)\mat{\Sigma}^{-1}: \left\{\mat{I}_m \otimes \frac{\vec{a}^\intercal (\dd\vec{a}) \vec{a}\vec{a}^\intercal}{\rho(\rho + \|\vec{a}\|^2)^2}\right\} \\
    &\qquad +  \mat{\Sigma}^{-1}(\vec{z}-\vec{x}) : \dd\vec{x} \\
    &= \left(\mat{Z} - \mat{\Sigma}\right) :\left(\mat{I}_m \otimes (\rho\mat{I}_m +\vec{a}\vec{a}^\intercal)\left(-\frac{(\dd\vec{a})\vec{a}^\intercal}{\rho(\rho + \|\vec{a}\|^2)}\right)(\rho\mat{I}_m +\vec{a}\vec{a}^\intercal)\right) \\
    &\qquad + \left(\mat{Z} - \mat{\Sigma}\right) : \left(\mat{I}_m \otimes (\rho\mat{I}_m +\vec{a}\vec{a}^\intercal)\left(\mat{I}_m \otimes \frac{\vec{a}^\intercal (\dd\vec{a}) \vec{a}\vec{a}^\intercal}{\rho(\rho + \|\vec{a}\|^2)^2}\right)(\rho\mat{I}_m +\vec{a}\vec{a}^\intercal)\right) \\
     &\qquad +  \mat{\Sigma}^{-1}(\vec{z}-\vec{x}) : \dd\vec{x} \\
     &= \left(\mat{Z} - \mat{\Sigma}\right) : \left(\mat{I}_m \otimes \frac{-1}{\rho(\rho + \|\vec{a}\|^2)}\left(\rho^2(\dd\vec{a})\vec{a}^\intercal + \rho\vec{a}^\intercal(\dd\vec{a})\vec{a}\vec{a}^\intercal + \rho \|\vec{a}\|^2(\dd\vec{a})\vec{a}^\intercal + \|\vec{a}\|^2\vec{a}^\intercal (\dd\vec{a})\vec{a}\vec{a}^\intercal\right)\right) \\
     &\qquad + \left(\mat{Z} - \mat{\Sigma}\right) : \left(\mat{I}_m \otimes \frac{\vec{a}^\intercal (\dd\vec{a})(\rho + \|\vec{a}\|^2)^2 }{\rho(\rho + \|\vec{a}\|^2)^2}\vec{a}\vec{a}^\intercal\right) \\
     &\qquad +  \mat{\Sigma}^{-1}(\vec{z}-\vec{x}) : \dd\vec{x} \\
     &=\left(\mat{Z} - \mat{\Sigma}\right) : \left(\mat{I}_m \otimes \frac{-1}{\rho(\rho + \|\vec{a}\|^2)}\left(\rho(\rho + \|\vec{a}\|^2)(\dd\vec{a})\vec{a}^\intercal + (\rho + \|\vec{a}\|^2)\vec{a}^\intercal (\dd\vec{a})\vec{a}\vec{a}^\intercal\right)\right) \\
     &\qquad + \left(\mat{Z} - \mat{\Sigma}\right) : \left(\mat{I}_m \otimes \frac{\vec{a}^\intercal (\dd\vec{a})(\rho + \|\vec{a}\|^2)^2 }{\rho(\rho + \|\vec{a}\|^2)^2}\vec{a}\vec{a}^\intercal\right) \\
     &\qquad +  \mat{\Sigma}^{-1}(\vec{z}-\vec{x}) : \dd\vec{x} \\
     &=\left(\mat{Z} - \mat{\Sigma}\right) : \left(\mat{I}_m \otimes \left(-(\dd\vec{a})\vec{a}^\intercal - \frac{\vec{a}^\intercal(\dd\vec{a})}{\rho}\vec{a}\vec{a}^\intercal\right)\right) + \left(\mat{Z} - \mat{\Sigma}\right) : \left(\mat{I}_m \otimes \frac{\vec{a}^\intercal (\dd\vec{a})}{\rho}\vec{a}\vec{a}^\intercal\right)\\
      &\qquad +  \mat{\Sigma}^{-1}(\vec{z}-\vec{x}) : \dd\vec{x} \\
      &= \left(\mat{Z} - \mat{\Sigma}\right) : \left(\mat{I}_m \otimes \left(-(\dd\vec{a})\vec{a}^\intercal \right)\right) + \mat{\Sigma}^{-1}(\vec{z}-\vec{x}) : \dd\vec{x} \,.\\
\end{align*}
Let $\mat{A}\in \R^{mn \times mn}$ be a symmetric matrix and $\mat{B}\in \R^{n\times n}$. It can be easily shown that
\begin{align*}
    \Trace(\mat{A}(\mat{I}_m \otimes \mat{B})) = \Trace((\mat{I}_m \otimes \mat{B})\mat{A}) = \sum_{i=1}^{m}\Trace(\mat{B}\mat{A}_{ii}) = \sum_{i=1}^{m}\Trace(\mat{A}_{ii}\mat{B}) = \Trace\left(\left(\sum_{i=1}^{m}\mat{A}_{ii}\right)\mat{B}\right)\,,
\end{align*}
where $\mat{A}_{ij}$ denotes the submatrix located at the $(i,j)$-th block
position when $\mat{A}$ is viewed as an $m\times m$ block matrix with
each block of size $n\times n$. Using the above, we have
\[
    \dd\mathcal{L} = -\sum_{i=1}^{m}\left(\mat{Z} - \mat{\Sigma}\right)_{ii}\vec{a} : (\dd\vec{a}) + \mat{\Sigma}^{-1}(\vec{z}-\vec{x}) : \dd\vec{x}\,.
\]

\section{An Alternative Approach for Covariance Modeling}
In this section, we explore an alternative approach for constructing the covariance of the search distribution. Rather than employing the Kronecker-factorized block diagonal structure used in \loren, we define a direct low-rank covariance structure given by
\[
\tilde{\mat{H}} = \mat{\Sigma}^{-1} = \rho\mat{I}_d + \vec{a}\vec{a}^\intercal\,,
\]
where $d=mn$ and $\vec{a}\in \R^d$. Although this formulation is computationally less efficient than the design of \loren, we present it to provide a broader context and to emphasize the efficiency and scalability of \loren’s covariance modeling.

If the precision matrix is modeled as a rank-1 plus the identity matrix, it can be expressed as
\begin{align*}
    \mat{\Sigma}^{-1} &= \rho\mat{I}_d + \vec{a}\vec{a}^\intercal = \sqrt{\rho}\left(\mat{I}_d + \frac{\vec{a}\vec{a}^\intercal}{\rho + \sqrt{\rho\|a\|^2 + \rho^2}}\right)\left(\mat{I}_d + \frac{\vec{a}\vec{a}^\intercal}{\rho + \sqrt{\rho\|a\|^2 + \rho^2}}\right)^\intercal\,,  \\
    \mat{\Sigma}^{-1}\vec{a} &= (\rho + \|\vec{a}\|^2)\vec{a}\,,  \\
    \mat{\Sigma}^{-1/2}\vec{a} &= \sqrt{\rho}\left(\vec{a} + \frac{\|\vec{a}\|^2}{\sqrt{\rho}(\sqrt{\rho}+ \sqrt{\rho + \|\vec{a}\|^2})}\vec{a}\right) = \frac{\sqrt{\rho}(\sqrt{\rho}+ \sqrt{\rho + \|\vec{a}\|^2}) + \|\vec{a}\|^2}{\sqrt{\rho} + \sqrt{\rho + \|\vec{a}\|^2}}\vec{a} \\
    &= \frac{\sqrt{\rho + \|\vec{a}\|^2}(\sqrt{\rho} + \sqrt{\rho + \|\vec{a}\|^2})}{\sqrt{\rho} + \sqrt{\rho + \|\vec{a}\|^2}}\vec{a} \\
    &= \left(\sqrt{\rho + \|\vec{a}\|^2}\right)\vec{a}\,. 
\end{align*}
By the Sherman-Morrison formula, we have
\begin{align*}
    \mat{\Sigma} &= \frac{1}{\rho}\left(\mat{I}_d - \frac{\vec{a}\vec{a}^\intercal}{\rho + \|\vec{a}\|^2}\right) \\
    &= \frac{1}{\sqrt{\rho}}\left(\mat{I}_d - \frac{\sqrt{\rho} + \sqrt{\rho + \|\vec{a}\|^2}}{\|\vec{a}\|^2\sqrt{\rho + \|\vec{a}\|^2}}\vec{a}\vec{a}^\intercal\right)\frac{1}{\sqrt{\rho}}\left(\mat{I}_d - \frac{\sqrt{\rho} + \sqrt{\rho + \|\vec{a}\|^2}}{\|\vec{a}\|^2\sqrt{\rho + \|\vec{a}\|^2}}\vec{a}\vec{a}^\intercal\right)^\intercal\,, \\
    \mat{\Sigma}\vec{a} &= \frac{1}{\rho + \|\vec{a}\|^2}\vec{a}\,, \\
    \mat{\Sigma}^{1/2}\vec{a} &= \frac{1}{\sqrt{\rho}}\left(\vec{a} - \frac{\sqrt{\rho} + \sqrt{\rho + \|\vec{a}\|^2}}{\sqrt{\rho + \|\vec{a}\|^2}}\vec{a}\right) = \frac{\sqrt{\rho + \|\vec{a}\|^2} - \sqrt{\rho} - \sqrt{\rho + \|\vec{a}\|^2}}{\sqrt{\rho}\sqrt{\rho + \|\vec{a}\|^2}} \\
    &= - \frac{1}{\sqrt{\rho + \|\vec{a}\|^2}}\vec{a}\,, \\
    \dd\mat{\Sigma} &= -\frac{(\dd\vec{a})\vec{a}^\intercal + \vec{a}(\dd\vec{a}^\intercal)}{\rho(\rho + \|\vec{a}\|^2)}+\frac{2\vec{a}^\intercal (\dd\vec{a}) \vec{a}\vec{a}^\intercal}{\rho(\rho + \|\vec{a}\|^2)^2}\,.
\end{align*}
To derive the gradient descent update equations, we begin by computing the derivatives, starting from Equation \eqref{eq:differntial}.
\begin{align*}
    \dd\mathcal{L} &= \frac{1}{2}\Trace\left(\dd\mat{\Sigma}\mat{\Sigma}^{-1}\left(\mat{Z} - \mat{\Sigma}\right)\mat{\Sigma}^{-1}\right) + (\dd\vec{x})^\intercal \mat{\Sigma}^{-1}(\vec{z}-\vec{x}) \\
    &= \frac{1}{2}\mat{\Sigma}^{-1}\left(\mat{Z} - \mat{\Sigma}\right)\mat{\Sigma}^{-1} : \dd\mat{\Sigma} +  \mat{\Sigma}^{-1}(\vec{z}-\vec{x}) : \dd\vec{x} \\
    &= \frac{1}{2} \mat{\Sigma}^{-1}\left(\mat{Z} - \mat{\Sigma}\right)\mat{\Sigma}^{-1} : \left(-\frac{(\dd\vec{a})\vec{a}^\intercal + \vec{a}(\dd\vec{a}^\intercal)}{\rho(\rho + \|\vec{a}\|^2)}+\frac{2\vec{a}^\intercal (\dd\vec{a}) \vec{a}\vec{a}^\intercal}{\rho(\rho + \|\vec{a}\|^2)^2}\right) + \mat{\Sigma}^{-1}(\vec{z}-\vec{x}) : \dd\vec{x} \\
    &= \mat{\Sigma}^{-1}\left(\mat{Z} - \mat{\Sigma}\right)\mat{\Sigma}^{-1} : \left(-\frac{(\dd\vec{a})\vec{a}^\intercal}{\rho(\rho + \|\vec{a}\|^2)}\right) + \mat{\Sigma}^{-1}\left(\mat{Z} - \mat{\Sigma}\right)\mat{\Sigma}^{-1}: \frac{\vec{a}^\intercal (\dd\vec{a}) \vec{a}\vec{a}^\intercal}{\rho(\rho + \|\vec{a}\|^2)^2} \\
    &\qquad +  \mat{\Sigma}^{-1}(\vec{z}-\vec{x}) : \dd\vec{x} \\
    &=-\frac{1}{\rho}\mat{\Sigma}^{-1}\left(\mat{Z} - \mat{\Sigma}\right)\vec{a} : \dd\vec{a} + \frac{(\dd\vec{a}^\intercal) \vec{a}}{\rho}\vec{a}^\intercal\left(\mat{Z} - \mat{\Sigma}\right)\vec{a} + \mat{\Sigma}^{-1}(\vec{z}-\vec{x}) : \dd\vec{x}\,.
\end{align*}
From the above, we have
\begin{align*}
    \grad_{\vec{a}}\log{p(\vec{z};\vec{\theta})} &= \frac{1}{\rho}\left(\mat{\Sigma}^{-1}\left(\mat{\Sigma} - \mat{Z}\right)\vec{a} + (\vec{a}^\intercal\left(\mat{Z} - \mat{\Sigma}\right)\vec{a})\vec{a}\right) \\
    &= \frac{1}{\rho}\left(\rho\left(\mat{\Sigma} - \mat{Z}\right)\vec{a} + (\vec{a}^\intercal\left(\mat{\Sigma} - \mat{Z}\right)\vec{a})\vec{a} + (\vec{a}^\intercal\left(\mat{Z} - \mat{\Sigma}\right)\vec{a})\vec{a}\right) \\
    &= \left(\mat{\Sigma} - \mat{Z}\right)\vec{a}\,, \\
    \grad_{\vec{x}}\log{p(\vec{z};\vec{\theta})} &= \mat{\Sigma}^{-1}(\vec{z}-\vec{x})\,.
\end{align*}
Applying the reparameterization trick $\vec{z} = \vec{x} + \mat{\Sigma}^{1/2}\vec{u}$ with $\vec{u}\sim \mathcal{N}(\vec{0}, \mat{I}_d)$ yields
\begin{align*}
    \grad_{\vec{a}}\log{p(\vec{z};\vec{\theta})} &= \frac{1}{\rho}\left(\mat{\Sigma}^{-1}(\mat{\Sigma}-\mat{\Sigma}^{1/2}\vec{u}\vec{u}^\intercal \mat{\Sigma}^{1/2})\vec{a} + (\vec{a}^\intercal(\mat{\Sigma}^{1/2}\vec{u}\vec{u}^\intercal \mat{\Sigma}^{1/2} - \mat{\Sigma})\vec{a})\vec{a}\right) \\
    &=\frac{1}{\rho}\left(\mat{\Sigma}^{-1/2}(\mat{I}_d - \vec{u}\vec{u}^\intercal)\mat{\Sigma}^{1/2}\vec{a} + (\vec{a}^\intercal \mat{\Sigma}^{1/2}(\vec{u}\vec{u}^\intercal - \mat{I}_d)\mat{\Sigma}^{1/2}\vec{a})\vec{a}\right) \\
    &=\frac{1}{\rho}\left(\mat{\Sigma}^{-1/2}\frac{(\vec{u}\vec{u}^\intercal - \mat{I}_d)\vec{a}}{\sqrt{\rho + \|\vec{a}\|^2}} + \frac{\vec{a}^\intercal(\vec{u}\vec{u}^\intercal - \mat{I}_d)\vec{a}}{\rho + \|\vec{a}\|^2}\vec{a}\right) \\
    &= \frac{1}{\rho}\left(\sqrt{\rho}\frac{(\vec{u}\vec{u}^\intercal - \mat{I}_d)\vec{a}}{\sqrt{\rho + \|\vec{a}\|^2}} + \frac{\vec{a}^\intercal(\vec{u}\vec{u}^\intercal - \mat{I}_d)\vec{a}}{\sqrt{\rho + \|\vec{a}\|^2}\left(\sqrt{\rho} + \sqrt{\rho + \|\vec{a}\|^2}\right)}\vec{a} + \frac{\vec{a}^\intercal(\vec{u}\vec{u}^\intercal - \mat{I}_d)\vec{a}}{\rho + \|\vec{a}\|^2}\vec{a}\right) \\
    &= \frac{1}{\rho(\rho + \|\vec{a}\|^2)}\left(\sqrt{\rho(\rho + \|\vec{a}\|^2)}(\vec{u}\vec{u}^\intercal - \mat{I}_d)\vec{a} + \frac{\left(\sqrt{\rho} + 2\sqrt{\rho + \|\vec{a}\|^2}\right)\vec{a}^\intercal(\vec{u}\vec{u}^\intercal - \mat{I}_d)\vec{a}}{\sqrt{\rho} + \sqrt{\rho + \|\vec{a}\|^2}}\vec{a}\right)\,, \\
    \grad_{\vec{x}}\log{p(\vec{z};\vec{\theta})} &= \mat{\Sigma}^{-1/2}\vec{u}\,.
\end{align*}
It is evident that this alternative approach for modeling the covariance of the search distribution incurs additional memory overhead by requiring the storage of $\vec{a}\in \R^d$, compared to $\vec{a}\in \R^n$ in \loren. Moreover, it demands increased computational cost for gradient estimation. Given that the primary objective of ZO optimization for fine-tuning LLMs is to minimize memory overhead relative to FO optimization, \loren naturally emerges as a efficient curvature-aware ZO method, making it a more suitable choice over this direct low-rank covariance approach.

\section{Decoupled Damping}
Ignoring the RLOO baseline, the parameter update for $\vec{x}$
can be simplified:
\begin{align*}
  \vec{x}
  &\gets \vec{x} - \eta \vec{g}(\vec{x}) =
    \vec{x} - \eta f\left(\vec{x} +
    \epsilon\overline{\mat{\Sigma}}^{1/2}\vec{u}\right)\epsilon^{-1}\overline{\mat{\Sigma}}^{1/2}\vec{u}
\\ 
  &= \vec{x} -
    \frac{\eta}{\sqrt{\rho}}
    f\left(\vec{x} +
    \epsilon\overline{\mat{\Sigma}}^{1/2}\vec{u}\right)
    \frac{1}{\epsilon}\left(\mat{I}_n
      - \kappa\vec{a}\vec{a}^\intercal\right)\vec{u}\,. 
\end{align*}
We observe that the effective step size $\eta/\sqrt{\rho}$ depends on
the damping factor $\rho$.
A small $\rho$ unintentionally increases
the effective step size, 
causing large and unstable
updates. Conversely, a large $\rho$ reduces the effective step size,
slowing down convergence and negatively impacting optimization
performance. To address 
this issue, we redefine the learning rate as $\eta' = \eta / \sqrt{\rho}$ such that it
inherently includes the $1/\sqrt{\rho}$ term, thus stabilizing the
effective update step size.
%
In this reformulation, the damping
parameter $\rho$ serves purely as a regularization term, ensuring
stable and appropriately scaled parameter updates. An analogous
decoupling is applied to the learning rate $\nu$ for covariance factor
updates, redefining it similarly to absorb the $1/\sqrt{\rho}$ factor
in Equation~\eqref{eq:grad_a}. 

\section{Convergence Analysis}
\label{sec:supp_convergence}
We make the following assumptions to establish the convergence
property of \loren.

\begin{assumption}[Smoothness] \label{assm:smoothness}
  The objective function $f:\R^{d}\to \R$ is $L$-smooth, meaning
  $\grad{f}$ satisfies
  \[
    \norm{\grad{f}(\vec{x}) - \grad{f}(\vec{y})}\leq
    L\norm{\vec{x}-\vec{y}}\,, \quad \forall \vec{x}, \vec{y}\in \R^{d}\,.
  \]
\end{assumption}
\begin{assumption}[Bounded Variance] \label{assm:bounded_variance}
  The variance of the stochastic gradient $\grad{f}(\vec{x}; \xi)$ is
  bounded by $\sigma^{2}$. That is,
  \[
    \E\left[\norm{\grad{f}(\vec{x};\xi) -
        \grad{f}(\vec{x})}^{2}\right]\leq \sigma^{2}\,, \quad \forall
    \vec{x}\in \R^{d}\,.
  \]
\end{assumption}

\subsection{Proof of Proposition~\ref{prop:grad_gauss_rs}}
We have
\begin{align*}
  \grad{f}_{\epsilon,\mat{\Sigma}}(\vec{x})
  &=\grad_{\vec{x}} \left\{%
    \E_{\vec{u}\sim\mathcal{N}(\vec{0},\mat{\Sigma})}\left[
    f(\vec{x} + \epsilon\vec{u})\right]\right\} \\
  &=(2\pi)^{-d/2}\det(\mat{\Sigma})^{-1/2}\grad_{\vec{x}}\int
    f(\vec{x}+\epsilon\vec{u})
    \exp(-\frac{1}{2}\vec{u}^{\intercal}\mat{\Sigma}^{-1}\vec{u})
    \dd{\vec{u}}\,. \\
  &=(2\pi)^{-d/2}\det(\mat{\Sigma})^{-1/2}\int
    \grad_{\vec{x}} f(\vec{x}+\epsilon\vec{u})
    \exp(-\frac{1}{2}\vec{u}^{\intercal}\mat{\Sigma}^{-1}\vec{u})
    \dd{\vec{u}}\,. \\
  \intertext{Using the change of variables $\vec{z} = \vec{x} +
  \epsilon\vec{u}$, we get }
  &=(2\pi)^{-d/2}\det(\mat{\Sigma})^{-1/2}\int
    \grad_{\vec{z}} f(\vec{z})\exp(-\frac{1}{2\epsilon^{2}}(\vec{z}-\vec{x})^{\intercal}\mat{\Sigma}^{-1}(\vec{z}-\vec{u}))     \abs*{\pdv{\vec{u}}{\vec{z}}} \dd{\vec{z}}\\
  &=(2\pi)^{-d/2}\det(\mat{\Sigma})^{-1/2}\epsilon^{-d}\int\grad_{\vec{z}}f(\vec{z})    
    \exp(-\frac{1}{2\epsilon^{2}}(\vec{z}-\vec{x})^{\intercal}\mat{\Sigma}^{-1}(\vec{z}-\vec{u}))
    \dd{\vec{z}}\,. \\  
  &=\E_{\vec{z}\sim\mathcal{N}(\vec{x},\epsilon^{2}\mat{\Sigma})}\left[%
    \grad{f}(\vec{z})\right] \\
  \intertext{By Stein's identity for the normal distribution, we have }
  &=\E_{\vec{z}\sim\mathcal{N}(\vec{x},\varepsilon^{2}\mat{\Sigma})}\left[%
    \epsilon^{-2}\mat{\Sigma}^{-1}(\vec{z}-\vec{x}) f(\vec{z})\right] \\
  &=(2\pi)^{-d/2}\det(\epsilon^{2}\mat{\Sigma})^{-1/2}\int
    \epsilon^{-2}\mat{\Sigma}^{-1}(\vec{z}-\vec{x})f(\vec{z})
    \exp(-\frac{1}{2\epsilon^{2}}(\vec{z}-\vec{x})^{\intercal}\mat{\Sigma}^{-1}(\vec{z}-\vec{x}))
    \dd{\vec{z}}\,.\\
  \intertext{Applying the change of variables $\vec{u} =
  \frac{\vec{z}-\vec{x}}{\epsilon}$ one more time gives}
  &=(2\pi)^{-d/2}\det(\epsilon^{2}\mat{\Sigma})^{-1/2}\epsilon^{-1}\int
    \mat{\Sigma}^{-1}\vec{u}f(\vec{x}+\epsilon\vec{u})
    \exp(-\frac{1}{2}\vec{u}\mat{\Sigma}^{-1}\vec{u})
    \abs*{\pdv{\vec{z}}{\vec{u}}} \dd{\vec{u}} \\
  &=(2\pi)^{-d/2}\det(\mat{\Sigma})^{-1/2}\varepsilon^{-1}\int
    \mat{\Sigma}^{-1}\vec{u}f(\vec{x}+\epsilon\vec{u})
    \exp(-\frac{1}{2}\vec{u}\mat{\Sigma}^{-1}\vec{u}) \dd{\vec{u}}\\
  &=\E_{\vec{u}\sim\mathcal{N}(\vec{0},\mat{\Sigma})}\left[
    \epsilon^{-1}f(\vec{x}+\epsilon\vec{u}) \mat{\Sigma}^{-1}\vec{u}\right]
\end{align*}
Similarly, we can show that $\grad{f}_{\epsilon,\mat{\Sigma}}(\vec{x})
= -\E_{\vec{u}\sim\mathcal{N}(\vec{0},\mat{\Sigma})}\left[
  \epsilon^{-1}f(\vec{x}-\epsilon\vec{u})
  \mat{\Sigma}^{-1}\vec{u}\right]$. Combining these two completes
the proof. Note that we also have
\[
  \grad{f}_{\epsilon,\mat{\Sigma}}(\vec{x}) =
  \E_{\vec{u}\sim\mathcal{N}(\vec{0},\mat{\Sigma})}\left[ \frac{f(\vec{x}+\epsilon\vec{u}) -
      f(\vec{x})}{\epsilon}\mat{\Sigma}^{-1}\vec{u}  \right]\,.
\]

\subsection{Bound on the gradient}
The RLOO gradient estimate is given by
\begin{equation} \label{eq:ca_grad_est}
  \tvec{g}_{\epsilon,\mat{\Sigma}}(\vec{x}_t)
  =\frac{1}{\epsilon(K-1)}\sum_{k=1}^{K}\left(f(\vec{x}+\epsilon\mat{\Sigma}^{1/2}\vec{u}_k;\xi) -
    \frac{1}{K}\sum_{j=1}^{K}f(\vec{x}+\epsilon\mat{\Sigma}^{1/2}\vec{u}_j;\xi)\right)
  \mat{\Sigma}^{1/2}\vec{u}_k\,. 
\end{equation}
We first show that the gradient estimate in~\eqref{eq:ca_grad_est} is
equal to the preconditioned gradient of random approximation of $f$.
$\tvec{g}_{\epsilon,\mat{\Sigma}}(\vec{x}_t)$.
\begin{align*}
  &\hspace{-1.5em}\E_{\vec{u}_{1:K},\xi}\left[\tvec{g}_{\epsilon,\mat{\Sigma}}(\vec{x})\right]\\
  &=\E_{\vec{u}_{1:K},\xi}\left[
    \frac{1}{\epsilon(K-1)}\sum_{k=1}^{K}\left(f(\vec{x}+\epsilon\mat{\Sigma}^{1/2}\vec{u}_k;\xi) -
    \frac{1}{K}\sum_{j=1}^{K}f(\vec{x}+\epsilon\mat{\Sigma}^{1/2}\vec{u}_j;\xi)\right)
    \mat{\Sigma}^{1/2}\vec{u}_k
    \right]\\
  &=\E_{\vec{u}_{1:K},\xi}\left[
    \frac{1}{\epsilon K}\sum_{k=1}^{K}\left(f(\vec{x}+\epsilon\mat{\Sigma}^{1/2}\vec{u}_k;\xi) -
    \frac{1}{K-1}\sum_{j\neq
    k}f(\vec{x}+\epsilon\mat{\Sigma}^{1/2}\vec{u}_j;\xi)\right)
    \mat{\Sigma}^{1/2}\vec{u}_k 
    \right] \\
  &=\frac{1}{\epsilon K}\sum_{k=1}^{K} \E_{\vec{u}_{1:K},\xi}\left[
    f(\vec{x}+\epsilon\mat{\Sigma}^{1/2}\vec{u}_k;\xi)\mat{\Sigma}^{1/2}\vec{u}_k\right]
  \\
  &\qquad 
    -\frac{1}{K-1}\sum_{j\neq  k}\E_{\vec{u}_{1:K},\xi}\left[
    f(\vec{x}+\epsilon\mat{\Sigma}^{1/2}\vec{u}_j)\mat{\Sigma}^{1/2}\vec{u}_k\right] \\
  &=\frac{1}{\epsilon K}\sum_{k=1}^{K} \E_{\vec{u}_{1:K},\xi}\left[
    f(\vec{x}+\epsilon\mat{\Sigma}^{1/2}\vec{u}_k;\xi)\mat{\Sigma}^{1/2}\vec{u}_k\right]
  \\
  &\qquad 
    -\frac{1}{K-1}\sum_{j\neq  k} \E_{\vec{u}_{j},\xi}\left[
    f(\vec{x}+\epsilon\mat{\Sigma}^{1/2}\vec{u}_j;\xi)\right]\cdot
    \E_{\vec{u}_{k}}\left[\mat{\Sigma}^{1/2}\vec{u}_k\right]\\
  &=\frac{1}{\epsilon K}\sum_{k=1}^{K} \E_{\vec{u}_{1:K},\xi}\left[
    f(\vec{x}+\epsilon\mat{\Sigma}^{1/2}\vec{u}_k;\xi)\mat{\Sigma}^{1/2}\vec{u}_k\right]
  \\
  &=\epsilon^{-1}\E_{\vec{u},\xi}\left[f(\vec{x}+\epsilon\mat{\Sigma}^{1/2}\vec{u};\xi)
    \mat{\Sigma}^{1/2}\vec{u}\right] \\
  &=
    \E_{\tvec{u},\xi}\left[\frac{f(\vec{x}+\epsilon\tvec{u};\xi)
    -f(\vec{x};\xi)}{\epsilon} \tvec{u}\right] =
    \tilde{\mat{H}}^{-1}\grad{f}_{\epsilon,\mat{\Sigma}}(\vec{x}) \,,
    \quad \tvec{u}\sim\mathcal{N}(\vec{0}, \mat{\Sigma})\,.
\end{align*}
From the first order Taylor approximation of $f$, we have
\[
  f(\vec{x}+\epsilon\mat{\Sigma}^{1/2}\vec{u}) = f(\vec{x}) +
  \epsilon\grad{f}(\vec{x})^{\intercal}\mat{\Sigma}^{1/2}\vec{u} +
  \mathcal{O}(\epsilon^{2})\,.
\]

\paragraph{Bound on the gradient norm} We can bound the norm of gradient
estimate as
\begin{align*}
  \E_{\vec{u}}\left[\norm{\tvec{g}_{\epsilon,\mat{\Sigma}}(\vec{x})}^{2}\right]
  &= \E_{\vec{u}}\left[%
    \norm{\mat{\Sigma}^{1/2}\vec{u}\vec{u}^{\intercal}\mat{\Sigma}^{1/2} \grad{f}(\vec{x})
    + \mathcal{O}(\epsilon)}^{2}
    \right] \\
  &\leq 2\E_{\vec{u}}\left[%
    \norm{\mat{\Sigma}^{1/2}\vec{u}\vec{u}^{\intercal}\mat{\Sigma}^{1/2}
    \grad{f}(\vec{x})}^{2}\right]
    + \mathcal{O}(\epsilon^{2}) \\
  &=2\E_{\vec{u}}\left[%
    \grad{f}(\vec{x})^{\intercal}\mat{\Sigma}^{1/2}\vec{u}\vec{u}^{\intercal}
    \mat{\Sigma}\vec{u}\vec{u}^{\intercal}\mat{\Sigma}^{1/2}\grad{f}(\vec{x})\right]
    + \mathcal{O}(\epsilon^{2}) \\
  &=2\E_{\vec{u}}\left[%
    (\vec{u}^{\intercal}\mat{\Sigma}^{1/2}\grad{f}(\vec{x}))^{2}\vec{u}^{\intercal}\mat{\Sigma}\vec{u}
    \right] + \mathcal{O}(\epsilon^{2}) \\
  &=2\Trace(\mat{\Sigma}) \cdot \grad{f}(\vec{x})^{\intercal}\mat{\Sigma}
    \grad{f}(\vec{x}) +
    4\grad{f}(\vec{x})^{\intercal}\mat{\Sigma}^{2}\grad{f}(\vec{x}) +
    \mathcal{O}(\epsilon^{2}) \\
  &\leq 2\left(\Trace(\mat{\Sigma}) + 2\rho^{-1}\right)
    \grad{f}(\vec{x})^{\intercal}\mat{\Sigma}\grad{f}(\vec{x})
    + \mathcal{O}(\epsilon^{2})\,,
\end{align*}
where the last equality is due to Lemma~\ref{lem:grad_norm_usu} and
the last inequality used the fact that the largest eigenvalue of
$\mat{\Sigma}$ is $\rho^{-1}$.
\begin{lemma} \label{lem:grad_norm_usu}
  For $\vec{u}\sim\mathcal{N}(\vec{0}, \mat{I}_d)$, a symmetric
  matrix $\mat{\Sigma}\in \R^{d\times d}$, and a fixed vector
  $\vec{v}\in \R^{d}$, we have
  \[
    \E_{\vec{u}}\left[(\vec{u}^{\intercal}\mat{\Sigma}^{1/2}\vec{v})^{2}\vec{u}\mat{\Sigma}\vec{u}\right]
    =\Trace(\mat{\Sigma}) \cdot \vec{v}^{\intercal}\mat{\Sigma}\vec{v}
    + 2\vec{v}^{\intercal}\mat{\Sigma}^{2}\vec{v}\,.
  \]
\end{lemma}
\begin{proof}
  Let $\vec{g} = \mat{\Sigma}^{1/2}\vec{v}$.
  \begin{align*}
    \E_{\vec{u}}\left[(\vec{u}^{\intercal}\mat{\Sigma}^{1/2}\vec{v})^{2}\vec{u}\mat{\Sigma}\vec{u}\right]
    &=\E\left[\left(\sum_{i=1}^{d}u_ig_i\right)^{2}
      \left(\sum_{k=1}^{d}u_k\sum_{\ell=1}^{d}\Sigma_{k\ell}u_\ell\right)
      \right] \\
    &=\E\left[\left(\sum_{i=1}^{d}\sum_{j=1}^{d}u_iu_jg_ig_j\right)
      \left(\sum_{k=1}^{d}\sum_{\ell=1}^{d}u_ku_\ell\Sigma_{k\ell}\right)\right]\\
    &=\E\left[\sum_{i,j,k,\ell}
      (u_iu_ju_ku_\ell)g_ig_j\Sigma_{k\ell}\right]\\
    &=\sum_{i,j,k,\ell} \E\left[u_iu_ju_ku_\ell\right]
      g_ig_j\Sigma_{k\ell}
    \\
    \intertext{Since the first and third moments of $\mathcal{N}(0, 1)$ are 0, we
    have}
    &=\sum_{i,j,k,\ell} (\delta_{ij}\delta_{k\ell} +
      \delta_{ik}\delta_{j\ell} + \delta_{i\ell}\delta_{jk})g_ig_j\Sigma_{k\ell}\,,
  \end{align*}
  where $\delta_{ij}$ is the Kronecker delta that returns 1 if $i=j$
  and 0 otherwise.
  \begin{enumerate}[label=(\roman*)]
  \item When $\delta_{ij}\delta_{k\ell} = 1$,
    \[
      \sum_{i,j,k,\ell}g_ig_j\Sigma_{k\ell} =
      \sum_{i,k}g_i^{2}\Sigma_{k,k} = \sum_{i}g_i^{2} \cdot
      \Trace(\mat{\Sigma}) = \norm{\vec{g}}^{2}\cdot \Trace(\mat{\Sigma})\,.
    \]
  \item When $\delta_{ik}\delta_{j\ell}=1$,
    \[
      \sum_{i,j,k,\ell}g_ig_j\Sigma_{k\ell} =
      \sum_{i,j}g_ig_j\Sigma_{ij} = \vec{g}^{\intercal}\mat{\Sigma}\vec{g}\,.
    \]
  \item When $\delta_{i\ell}\delta_{jk}=1$,
    \[
      \sum_{i,j,k,\ell}g_ig_j\Sigma_{k\ell} =
      \sum_{i,j}g_ig_j\Sigma_{ji} = \vec{g}^{\intercal}\mat{\Sigma}\vec{g}\,.
    \]
  \end{enumerate}  
  Combining these three cases gives the result.
\end{proof}
For a positive definite matrix $\mat{\Sigma}$, we define $\mat{\Sigma}$-norm of $\vec{x}$ as 
$\norm{\vec{x}}_{\mat{M}} = \sqrt{\vec{x}^\intercal\mat{M}\vec{x}}$.
The $\mat{\Sigma}$-norm of the stochastic gradient is
\begin{align*}
  \E\left[\norm{\grad{f}(\vec{x}_t;\xi_t)}^{2}_{\mat{\Sigma}_t}\right]
  &=\E\left[\norm{\grad{f}(\vec{x}_t;\xi_t) +
    \grad{f}(\vec{x}_t)-\grad{f}(\vec{x}_t)}^{2}_{\mat{\Sigma}_t}\right]
  \\
  &\leq 2\norm{\grad{f}(\vec{x}_t)}_{\mat{\Sigma}_t}^{2}
    + 2
    \E\left[\norm{\grad{f}(\vec{x}_t;\xi_t)-\grad{f}(\vec{x}_t)}^{2}_{\mat{\Sigma}_t}\right]\\
  &\leq 2\norm{\grad{f}(\vec{x}_t)}_{\mat{\Sigma}_t}^{2}
    +
    2\rho^{-1}\E\left[\norm{\grad{f}(\vec{x}_t;\xi_t)-\grad{f}(\vec{x}_t)}^{2}\right]\\
  &\leq 2\norm{\grad{f}(\vec{x}_t)}_{\mat{\Sigma}_t}^{2} +
    2\rho^{-1}\sigma^{2}\,.
\end{align*}

\subsection{Proof of Theorem~\ref{thm:convergence}}
From the $L$-smoothness of $f$,
\begin{align*}
  \E_{\vec{u}}\left[f(\vec{x}_{t+1};\xi_{t+1})\right]
  &\leq f(\vec{x}_{t};\xi_t) -\eta_t\E_{\vec{u}}\left[
    \grad{f}(\vec{x}_t;\xi_t)^{\intercal}  \tvec{g}_{\epsilon,\mat{\Sigma}}(\vec{x}_t)\right] +
    \frac{L\eta_t^{2}}{2}
    \E_{\vec{u}}\left[\norm{\tvec{g}_{\epsilon,\mat{\Sigma}}(\vec{x}_t)}^{2}\right]
  \\
  &\leq f(\vec{x}_t; \xi_t) -
    \eta_t\norm{\grad{f}(\vec{x}_t; \xi_t)}^{2}_{\mat{\Sigma}}
    + \eta_t\mathcal{O}(\epsilon\norm{\grad{f}(\vec{x}_t; \xi_t)}) \\
  &\qquad
    + 2\eta_t^{2}L(\Trace(\mat{\Sigma}_t) + 2\rho^{-1})\norm{\grad{f}(\vec{x}_t; \xi_t)}_{\mat{\Sigma}_t}^{2} \\
  &\qquad + 2\eta_t^{2}L(\Trace(\mat{\Sigma}_t) +2\rho^{-1})\rho^{-1}\sigma^{2} + \mathcal{O}(\epsilon^{2})\\
  &\leq f(\vec{x}_t; \xi_t)
    -\frac{\eta_t}{2}\norm{\grad{f}(\vec{x}_t;\xi_t)}^{2}_{\mat{\Sigma}_t}
    + 2\eta_t^{2}L(\Trace(\mat{\Sigma}) +
    2\rho^{-1})\norm{\grad{f}(\vec{x}_t; \xi_t)}_{\mat{\Sigma}_t}^{2}
  \\
  &\qquad
    + 2\eta_t^{2}L(\Trace(\mat{\Sigma}_t)
    +2\rho^{-1})\rho^{-1}\sigma^{2} + \mathcal{O}(\epsilon^{2})\\
  &=f(\vec{x}_t; \xi_t)-\frac{\eta_t}{2}\left(1 - 4\eta_tL(\Trace(\mat{\Sigma}_t) +
    2\rho^{-1})\right)\norm{\grad{f}(\vec{x}_t;\xi_t)}^{2}_{\mat{\Sigma}_t}
  \\
  &\qquad + 2\eta_t^{2}L(\Trace(\mat{\Sigma}_t) + 2\rho^{-1})
    \rho^{-1}\sigma^{2} + \mathcal{O}(\epsilon^{2})\,. \\
  \intertext{With the choice of $\eta_t = \eta = \frac{1}{8L\sqrt{T}(\max_t
  \Trace(\mat{\Sigma}_t) + 2\rho^{-1})}$, we have}
  &\leq f(\vec{x}_t; \xi_t)-\frac{\eta_t}{4}\norm{\grad{f}(\vec{x}_t;\xi_t)}^{2}_{\mat{\Sigma}_t}
  + 2\eta_t^{2}L(\Trace(\mat{\Sigma}_t)+2\rho^{-1})
    \rho^{-1}\sigma^{2} + \mathcal{O}(\epsilon^{2})\,.
\end{align*}
Rearranging the equation yields
\begin{align*}
  \E\left[\norm{\grad{f}(\vec{x}_t;\xi_t)}^{2}_{\mat{\Sigma}_t}\right]
  &\leq \frac{4\E\left[f(\vec{x}_t; \xi_t) - f(\vec{x}_{t+1};\xi_{t+1})\right]}{\eta_t}
    + 8\eta_tL(\Trace(\mat{\Sigma}_t) +
    2\rho^{-1})\frac{\sigma^{2}}{\rho}
    + \mathcal{O}(\epsilon^{2})\\
  \intertext{Summing the equations for $t=1, 2, \ldots, T$, we obtain}
  \E\left[\sum_{t=1}^{T}
  \norm{\grad{f}(\vec{x}_t;\xi_t)}_{\mat{\Sigma}_t}^{2}\right]
  &\leq \frac{4\left(f(\vec{x}_1; \xi_1) -
    f(\vec{x}_{T+1};\xi_{T+1})\right)}{\eta}
    +\frac{\sigma^{2}\sqrt{T}}{\rho}
    + \mathcal{O}(T\epsilon^{2}) \\
  &\leq \frac{4\left(f(\vec{x}_1; \xi_1) - f(\vec{x}_{*};\xi_{*})\right)}{\eta}
    +\frac{\sigma^{2}\sqrt{T}}{\rho}
    + \mathcal{O}(T\epsilon^{2})\,. \\
  \intertext{From the above, we get}
  \min_{t=1:T}\E\left[\norm{\grad{f}(\vec{x}_t;\xi_t)}^{2}\right]
  &\leq
    \frac{1}{T}\E\left[\sum_{t=1}^{T}\norm{\grad{f}(\vec{x}_t;\xi_t)}^{2}\right]
    \leq \frac{1}{T\alpha_{\min}}\E\left[\sum_{t=1}^{T}
    \norm{\grad{f}(\vec{x}_t;\xi_t)}_{\mat{\Sigma}_t}^{2}\right]\\
  &\leq \frac{4\left(f(\vec{x}_1; \xi_1) - f(\vec{x}_{*};\xi_{*})\right)}{T\eta\alpha_{\min}} +
    \frac{\sigma^{2}}{\sqrt{T}\alpha_{\min}} +
    \mathcal{O}(\epsilon^{2})\\
  &=\frac{32L(\min_t \Trace(\mat{\Sigma}_t) +
    2\rho^{-1})\left(f(\vec{x}_1; \xi_1) -
    f(\vec{x}_{*};\xi_{*})\right)}{\sqrt{T}\alpha_{\min}}
    + \frac{\sigma^{2}}{\sqrt{T}\alpha_{\min}} +
    \mathcal{O}(\epsilon^{2})\,,
\end{align*}
where $\alpha_{\min} = (\rho + \max_t\norm{\vec{a}_t}^{2})^{-1}$ is the smallest
eigenvalue of $\mat{\Sigma}_t$.

\section{Additional Experimental Results}

\subsection{Training Loss Curves on GLUE Benchmarks} \label{apdx:glue_train_loss}
\begin{figure}[b]
    \centering
    \begin{subfigure}{0.245\textwidth}
        \centering
        \includegraphics[width=\textwidth]{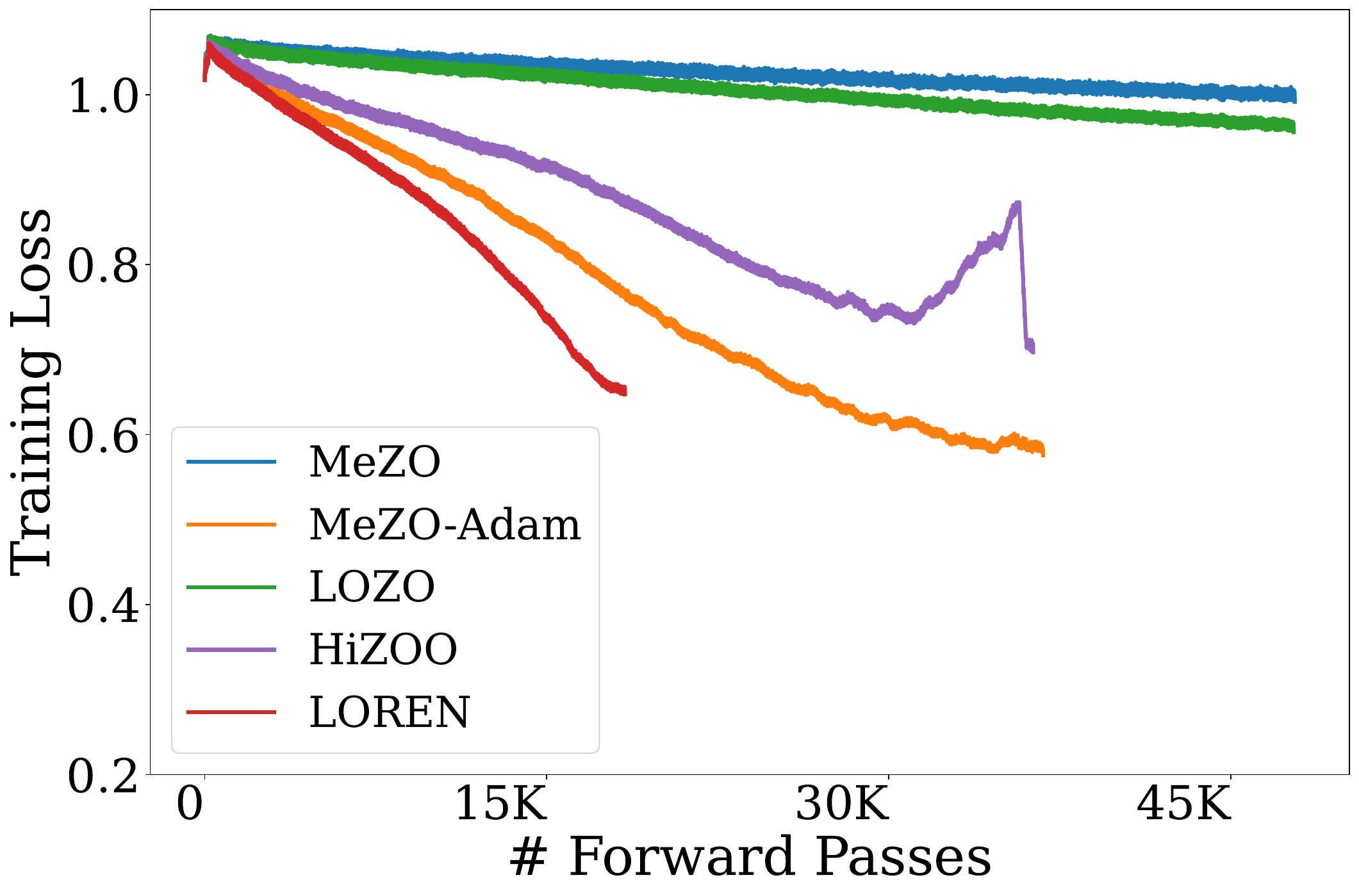}
        \subcaption{MNLI}
        \label{fig:train_loss_mnli}
    \end{subfigure}
    \begin{subfigure}{0.245\textwidth}
        \centering
        \includegraphics[width=\textwidth]{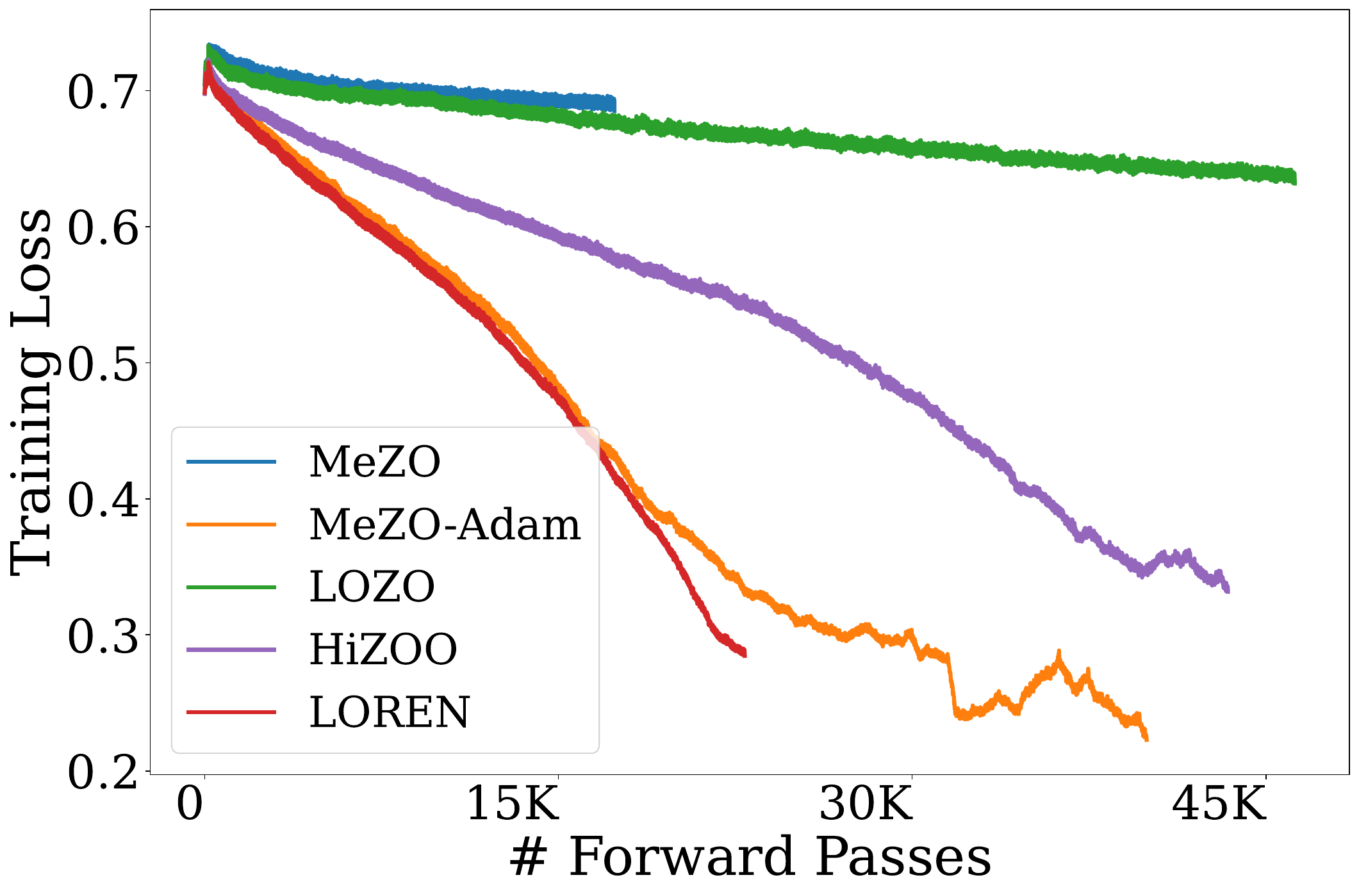}
        \subcaption{QNLI}
        \label{fig:train_loss_qnli}
    \end{subfigure}
    \begin{subfigure}{0.245\textwidth}
        \centering
        \includegraphics[width=\textwidth]{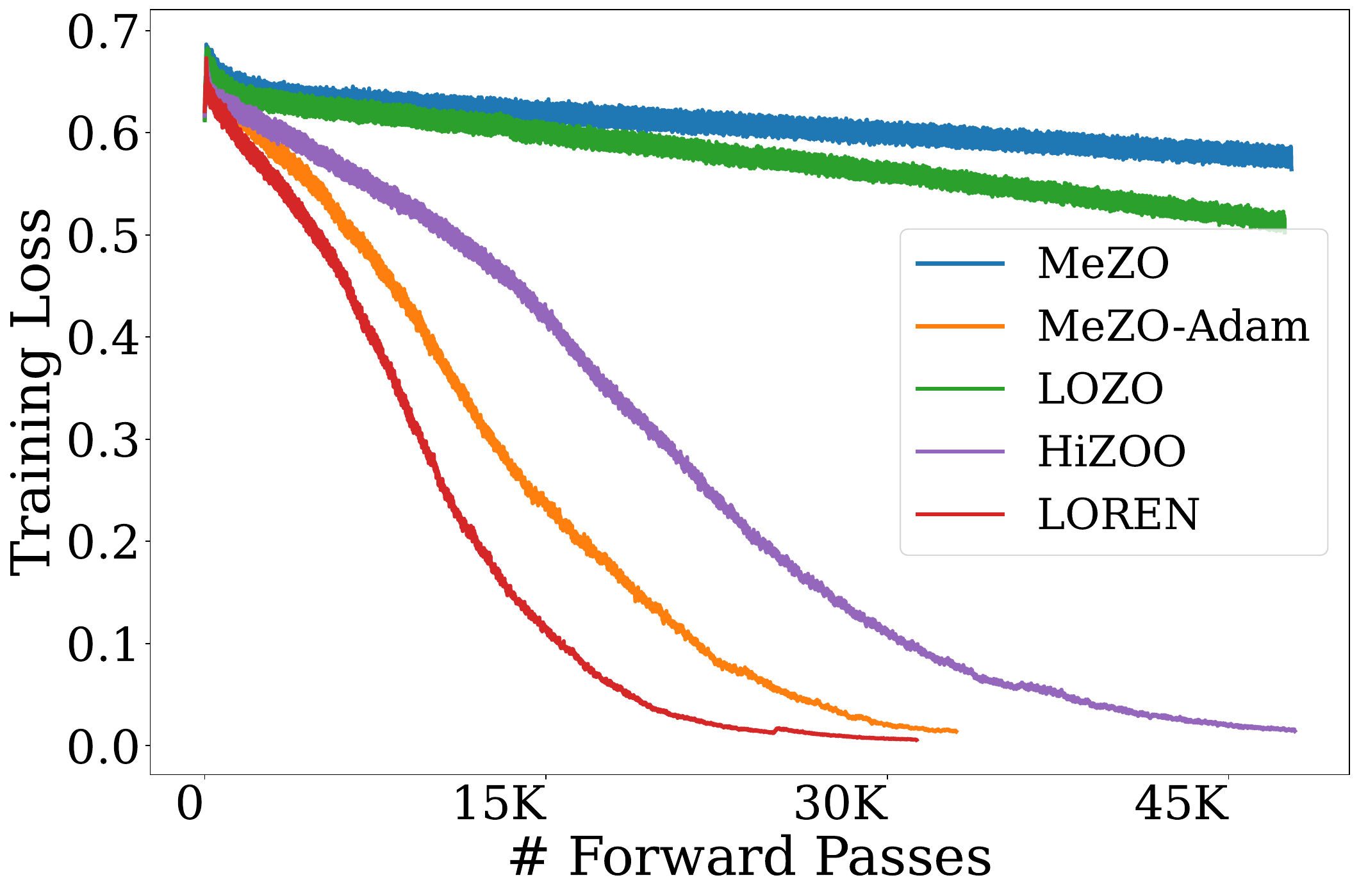}
        \subcaption{SST-2}
        \label{fig:train_loss_sst2}
    \end{subfigure}
    \begin{subfigure}{0.245\textwidth}
        \centering
        \includegraphics[width=\textwidth]{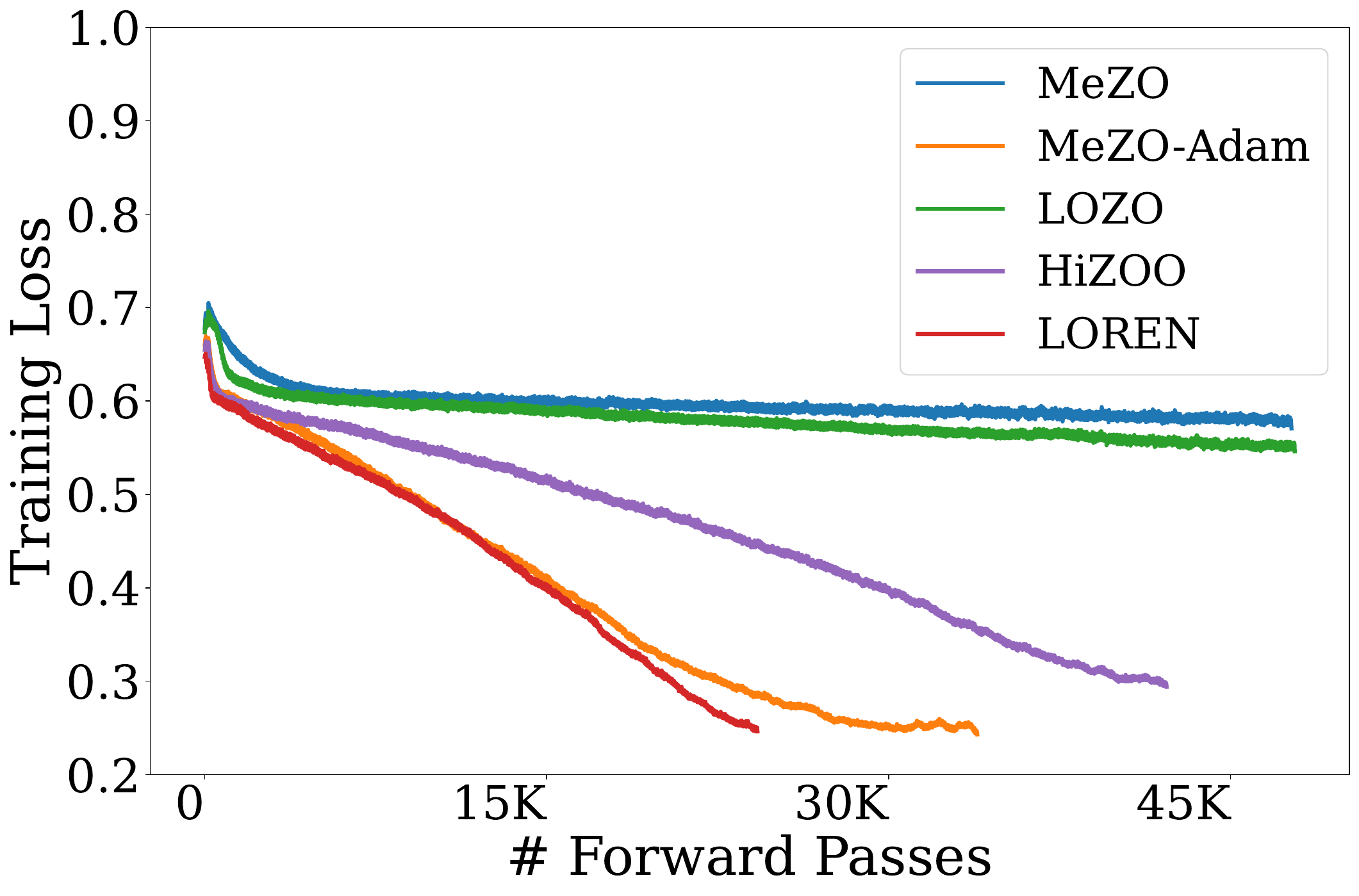}
        \subcaption{CoLA}
        \label{fig:train_loss_cola}
    \end{subfigure}
    \caption{Training loss curves for different ZO optimizers when fine-tuning GPT-2-XL on GLUE tasks.}
    \label{fig:train_loss_glue}
\end{figure}

In Figure~\ref{fig:train_loss_glue}, which depicts GPT-2-XL fine-tuning on GLUE, \loren’s curve drops more sharply than any other ZO method and attains its lowest loss in roughly half the evaluations needed by LOZO or HiZOO.

\subsection{Fine-tuning using Default Number of Forward Passes} \label{apdx:2fwd_pass}

Table~\ref{tab:2pass_result} presents results when each ZO optimizer is run using its own default number of forward passes per step—typically two for MeZO, MeZO-Adam, and LOZO, three for MeZO-SVRG (on average when $q=2$) and HiZOO. Although these settings reflect how each optimizer is commonly configured, they lead to degraded accuracy across most tasks compared to results under a standardized 6-pass budget. This degradation is expected, as fewer forward passes produce noisier gradient estimates. For instance, MeZO-Adam and HiZOO show substantial performance drops, particularly on RoBERTa-large, where MeZO-Adam’s average accuracy decreases by more than 13 points. Thus we adopt the same 6-pass setup across all baselines in our main experiments to ensure a fair and consistent comparison.

\begin{table}[tb]
    \centering
    \caption{Experimental results on DistilBERT and RoBERTa using each optimizer's default number of forward passes per step. Reported metrics include best accuracy (\%) with standard deviation over 5 runs and the averaged accuracy across 4 benchmark tasks from GLUE.}
    \footnotesize
    \begin{tabular}{lcccccc}
        \toprule
        & \multicolumn{6}{c}{\textbf{DistilBERT (66M) --- FP32}} \\
        \cmidrule(lr){2-6}
        Task & MNLI & QNLI & SST-2 & CoLA & Avg & 6-Pass Avg  \\ 
        \midrule
        MeZO & 39.8\scriptsize{$\pm$0.0} & 48.6\scriptsize{$\pm$0.8} & 61.9\scriptsize{$\pm$0.7} & 67.0\scriptsize{$\pm$0.3} & 54.3 (-0.1) & 54.4\\
        MeZO-Adam & 40.4\scriptsize{$\pm$0.5} & 69.4\scriptsize{$\pm$2.1} & 77.8\scriptsize{$\pm$0.8} & 66.4\scriptsize{$\pm$0.3} & 63.5 (-1.1) & 64.6\\
        MeZO-SVRG & 42.7\scriptsize{$\pm$1.1} & 65.6\scriptsize{$\pm$1.4} & 73.8\scriptsize{$\pm$2.2} & 65.8\scriptsize{$\pm$0.3} & 61.9 (+0.3) & 61.6 \\
        LOZO & 39.9\scriptsize{$\pm$0.1} & 51.6\scriptsize{$\pm$1.2} & 61.8\scriptsize{$\pm$0.7} & 66.0\scriptsize{$\pm$0.3} & 54.8 (-0.5)  & 55.3\\
        HiZOO & 39.9\scriptsize{$\pm$0.1} & 64.7\scriptsize{$\pm$4.7} & 76.5\scriptsize{$\pm$1.1} & 66.7\scriptsize{$\pm$0.8} & 62.0 (-0.6) & 62.6\\
        \midrule
        \loren & 39.8\scriptsize{$\pm$0.0} & 73.0\scriptsize{$\pm$2.0} & 81.7\scriptsize{$\pm$1.0} & 67.2\scriptsize{$\pm$0.8} & \textbf{65.4} & -- \\
        \midrule
        & \multicolumn{6}{c}{\textbf{RoBERTa‐large (355M) --- FP32}} \\
        \cmidrule(lr){2-6}
        Task & MNLI & QNLI & SST-2 & CoLA & Avg & 6-Pass Avg   \\  
        \midrule
        MeZO & 40.0\scriptsize{$\pm$0.1} & 73.3\scriptsize{$\pm$0.7} & 54.6\scriptsize{$\pm$0.5} & 66.7\scriptsize{$\pm$0.5} & 58.6 (+0.2) & 58.4 \\
        MeZO‐Adam & 43.0\scriptsize{$\pm$4.0} & 77.0\scriptsize{$\pm$2.5} & 54.1\scriptsize{$\pm$0.4} & 67.1\scriptsize{$\pm$0.4} & 60.3 (-13.3) & \textbf{73.6} \\
        MeZO‐SVRG & 39.5\scriptsize{$\pm$0.0} & 56.7\scriptsize{$\pm$2.5} & 55.2\scriptsize{$\pm$0.1} & 67.1\scriptsize{$\pm$0.5} & 54.6 (+0.0) & 54.6\\
        LOZO & 41.9\scriptsize{$\pm$3.4} & 69.4\scriptsize{$\pm$3.7} & 53.1\scriptsize{$\pm$0.0} & 70.9\scriptsize{$\pm$2.2} & 58.8 (-0.3) & 59.1\\
        HiZOO & 44.1\scriptsize{$\pm$1.9} & 64.8\scriptsize{$\pm$1.9} & 63.5\scriptsize{$\pm$1.9} & 67.7\scriptsize{$\pm$0.2} & 60.0 (-4.2) & 64.2\\
        \midrule
        \loren & 44.3\scriptsize{$\pm$1.4} & 76.3\scriptsize{$\pm$1.5} & 86.1\scriptsize{$\pm$3.1} & 73.8\scriptsize{$\pm$0.4} & \textbf{70.1} & --\\
        \bottomrule
    \end{tabular}
    \label{tab:2pass_result}
\end{table}

\section{Ablation Study} \label{apdx:ablation}
We conduct an ablation study to assess \loren's sensitivity to three key hyperparameters: 
(i) the learning rate $\nu$ for the covariance parameter $\vec{a}$, 
(ii) the damping parameter $\rho$, and 
(iii) the number of forward passes $K$ per iteration. 
We evaluated the test accuracy of the RoBERTa-large model on the QNLI tasks across 5 independent runs under various configurations.

\begin{figure}[tb]
    \centering
    \includegraphics[width=0.95\textwidth]{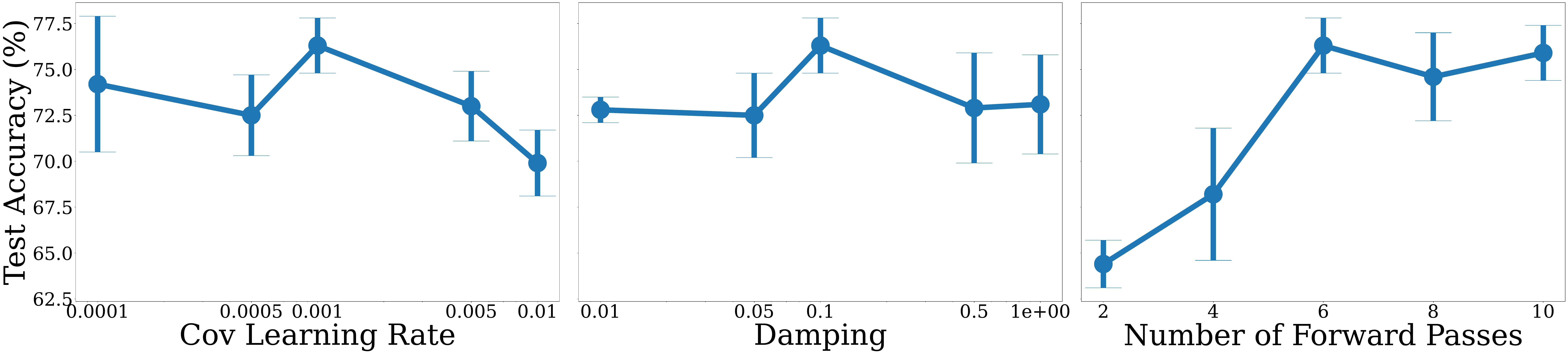}
    \caption{Fine-tuning results of QNLI task on RoBERTa-large with varying (Left) covariance learning rate, (Center) damping, and (Right) number of forward passes.}
    \label{fig:ablation}
\end{figure}

Figure~\ref{fig:ablation} presents the results. For the learning rate $\nu$, test accuracy peaks at $\nu = 0.001$ but declines beyond this point, indicating a trade-off between adaptation speed and model performance. The damping parameter $\rho$ achieves optimal performance at $\rho = 0.1$, with lower and higher values limiting adaptation. Finally, test accuracy consistently improves with an increasing number of forward passes $K$ until $K=6$, beyond which it plateaus, reflecting diminishing returns.

\section{Experimental Details} 

\subsection{Datasets and Implementation} \label{apdx:details}
Following \cite{gautam2024variancereduced}, we focus on fine-tuning LLMs for text classification tasks using datasets from the GLUE and SuperGLUE benchmarks. Specifically, we use full-precision (FP32) for DistilBERT, RoBERTa-large, GPT-2-XL, and OPT-2.7B, and half-precision (BF16) for LLaMA-3-8B and OPT-13B to accommodate GPU memory constraints.

We trained on 512 randomly sampled examples and evaluate on 256 validation examples, reporting validation accuracy as a proxy for test accuracy since test labels for both GLUE and SuperGLUE tasks are unavailable. Early stopping was applied, given that ZO optimizers generally exhibit diminishing returns in performance with increased iterations after convergence. For consistency, we set the number of forward passes to 6 across all ZO optimizers, aligning with \loren’s optimal configuration for RLOO gradient estimation.

\subsection{Hyperparameter Configurations}
We present the hyperparameter configurations used for fine-tuning the six language models (DistilBERT, RoBERTa-large, GPT-2-XL, OPT-2.7B, LLaMA-3-8B, and OPT-13B). Each table below provides detailed hyperparameter settings for each ZO optimizer, including MeZO, MeZO-Adam, MeZO-SVRG, LOZO, HiZOO, and \loren. The configurations were carefully selected through grid search, and the bold values indicate the settings used to generate the final results.

Table~\ref{tab:hyperpar_distilbert} summarizes the hyperparameter settings for fine-tuning DistilBERT, where the batch size, learning rate, perturbation smoothing, total steps, and other important parameters were optimized for each optimizer. Similar configurations were applied to RoBERTa-large, GPT-2-XL, OPT-2.7B, LLaMA-3-8B, and OPT-13B with the specific settings provided in Tables~\ref{tab:hyperpar_roberta},\ref{tab:hyperpar_gpt}, \ref{tab:hyperpar_opt}, and \ref{tab:hyperpar_llama} respectively. These hyperparameter settings ensure a fair comparison across all ZO optimizers, allowing each method to fully leverage its algorithmic strengths.

\begin{table}[ht]
\centering
\caption{The hyperparameter configurations used for fine-tuning DistilBERT, with bold values indicating the settings applied to generate the final results.}
\begin{tabular}{llc}
\toprule
\textbf{Algorithm} & \textbf{Hyperparameters} & \textbf{Values} \\
\midrule
\multirow{4}{*}{MeZO} & Batch size & $64$ \\
                      & Learning rate & $\{1e-4, 5e-5, \mathbf{1e-5}, 5e-6, 1e-6\}$ \\
                      & $\epsilon$ & $1e-3$ \\
                      & Total Steps & $24,000$ \\
\midrule
\multirow{5}{*}{MeZO-Adam} & Batch size & $64$ \\
                      & Learning rate & $\{1e-3, 5e-4, \mathbf{1e-4}, 5e-5, 1e-5\}$ \\
                      & Betas & $(0.9, 0.999)$ \\
                      & $\epsilon$ & $1e-3$ \\
                      & Total Steps & $24,000$ \\
\midrule
\multirow{6}{*}{MeZO-SVRG} & Batch size & $64$ \\
                           & Learning rate (Full-batch) & $\{1e-2, 5e-3, \mathbf{1e-3}, 5e-4, 1e-4\}$\\
                           & Learning rate (Mini-batch) & $\{1e-5, 5e-6, \mathbf{1e-6}, 5e-7, 1e-7\}$ \\
                           & $\epsilon$ & $1e-3$ \\
                           & Frequency of Full-batch Update & $2$ \\
                           & Total Steps & $24,000$ \\
\midrule
\multirow{6}{*}{LOZO} & Batch size & $64$ \\
                      & Learning rate & $\{1e-4, 5e-5, \mathbf{1e-5}, 5e-6, 1e-6\}$ \\
                      & Rank & $\{\mathbf{2}, 4, 8\}$ \\
                      & Interval & $\{\mathbf{50}, 100\}$ \\
                      & $\epsilon$ & $1e-3$ \\
                      & Total Steps & $24,000$ \\
\midrule
\multirow{5}{*}{HiZOO} & Batch size & $64$ \\
                      & Learning rate & $\{1e-3, 5e-4, \mathbf{1e-4}, 5e-5, 1e-5\}$ \\
                      & Hessian Smoothing & $1e-8$ \\
                      & $\epsilon$ & $1e-3$ \\
                      & Total Steps & $24,000$ \\ 
\midrule
\multirow{6}{*}{\loren} & Batch size & $64$ \\
                      & Learning rate ($\eta$) & $\{1e-4, 5e-5, \mathbf{1e-5}, 5e-6, 1e-6\}$ \\
                      & Learning rate ($\nu$) & $1e-3$ \\
                      & Damping & $1e-2$ \\
                      & $\epsilon$ & $1e-3$ \\
                      & Total Steps & $24,000$ \\    
\bottomrule
\end{tabular}
\label{tab:hyperpar_distilbert}
\end{table}

\begin{table}[tb]
\centering
\caption{The hyperparameter configurations used for fine-tuning RoBERTa-large, with bold values indicating the settings applied to generate the final results.}
\begin{tabular}{llc}
\toprule
\textbf{Algorithm} & \textbf{Hyperparameters} & \textbf{Values} \\
\midrule
\multirow{4}{*}{MeZO} & Batch size & $64$ \\
                      & Learning rate & $\{1e-4, 5e-5, \mathbf{1e-5}, 5e-6, 1e-6\}$ \\
                      & $\epsilon$ & $1e-3$ \\
                      & Total Steps & $20,000$ \\
\midrule
\multirow{5}{*}{MeZO-Adam} & Batch size & $64$ \\
                      & Learning rate & $\{1e-3, 5e-4, 1e-4, \mathbf{5e-5}, 1e-5\}$ \\
                      & Betas & $(0.9, 0.999)$ \\
                      & $\epsilon$ & $1e-3$ \\
                      & Total Steps & $20,000$ \\
\midrule
\multirow{6}{*}{MeZO-SVRG} & Batch size & $64$ \\
                           & Learning rate (Full-batch) & $\{1e-4, 5e-5, \mathbf{1e-5}, 5e-6, 1e-6\}$\\
                           & Learning rate (Mini-batch) & $\{1e-5, 5e-6, \mathbf{1e-6}, 5e-7, 1e-7\}$ \\
                           & $\epsilon$ & $1e-3$ \\
                           & Frequency of Full-batch Update & $2$ \\
                           & Total Steps & $20,000$ \\
\midrule
\multirow{6}{*}{LOZO} & Batch size & $64$ \\
                      & Learning rate & $\{1e-4, 5e-5, \mathbf{1e-5}, 5e-6, 1e-6\}$ \\
                      & Rank & $\{\mathbf{2}, 4, 8\}$ \\
                      & Interval & $\{\mathbf{50}, 100\}$ \\
                      & $\epsilon$ & $1e-3$ \\
                      & Total Steps & $20,000$ \\
\midrule
\multirow{5}{*}{HiZOO} & Batch size & $64$ \\
                      & Learning rate & $\{1e-3, 5e-4, 1e-4, \mathbf{5e-5}, 1e-5\}$ \\
                      & Hessian Smoothing & $1e-8$ \\
                      & $\epsilon$ & $1e-3$ \\
                      & Total Steps & $20,000$ \\ 
\midrule
\multirow{6}{*}{\loren} & Batch size & $64$ \\
                      & Learning rate ($\eta$) & $\{1e-4, 5e-5, 1e-5, \mathbf{5e-6}, 1e-6\}$ \\
                      & Learning rate ($\nu$) & $1e-3$ \\
                      & Damping & $1e-1 $ \\
                      & $\epsilon$ & $1e-3$ \\
                      & Total Steps & $20,000$ \\    
\bottomrule
\end{tabular}
\label{tab:hyperpar_roberta}
\end{table}

\begin{table}[tb]
\centering
\caption{The hyperparameter configurations used for fine-tuning GPT-2-XL, with bold values indicating the settings applied to generate the final results.}
\begin{tabular}{llc}
\toprule
\textbf{Algorithm} & \textbf{Hyperparameters} & \textbf{Values} \\
\midrule
\multirow{4}{*}{MeZO} & Batch size & $64$ \\
                       & Learning rate & $\{1e-4, 5e-5, 1e-5, \mathbf{5e-6}, 1e-6\}$ \\
                      & $\epsilon$ & $1e-3$ \\
                      & Total Steps & $8,000$ \\
\midrule
\multirow{5}{*}{MeZO-Adam} & Batch size & $64$ \\
                      & Learning rate & $\{1e-3, 5e-4, \mathbf{1e-4}, 5e-5, 1e-5\}$ \\
                      & Betas & $(0.9, 0.999)$ \\
                      & $\epsilon$ & $1e-3$ \\
                      & Total Steps & $8,000$ \\
\midrule
\multirow{6}{*}{MeZO-SVRG} & Batch size & $64$ \\
                           & Learning rate (Full-batch) & $\{1e-4, \mathbf{5e-5}, 1e-5, 5e-6, 1e-6\}$\\
                           & Learning rate (Mini-batch) & $\{1e-5, 5e-6, \mathbf{1e-6}, 5e-7, 1e-7\}$ \\
                           & $\epsilon$ & $1e-3$ \\
                           & Frequency of Full-batch Update & $2$ \\
                           & Total Steps & $8,000$ \\
\midrule
\multirow{6}{*}{LOZO} & Batch size & $64$ \\
                      & Learning rate & $\{1e-4, 5e-5, 1e-5, \mathbf{5e-6}, 1e-6\}$ \\
                      & Rank & $\{\mathbf{2}, 4, 8\}$ \\
                      & Interval & $\{\mathbf{50}, 100\}$ \\
                      & $\epsilon$ & $1e-3$ \\
                      & Total Steps & $8,000$ \\
\midrule
\multirow{5}{*}{HiZOO} & Batch size & $64$ \\
                      & Learning rate & $\{1e-3, 5e-4, 1e-4, \mathbf{5e-5}, 1e-5\}$ \\
                      & Hessian Smoothing & $1e-8$ \\
                      & $\epsilon$ & $1e-3$ \\
                      & Total Steps & $8,000$ \\
\midrule
\multirow{6}{*}{\loren} & Batch size & $64$ \\
                      & Learning rate ($\eta$) & $\{\mathbf{1e-5}, 5e-6, 1e-6, 5e-7, 1e-7\}$ \\
                      & Learning rate ($\nu$) & $1e-3$ \\
                      & Damping & $1e-1$ \\
                      & $\epsilon$ & $1e-3$ \\
                      & Total Steps & $8,000$ \\    
\bottomrule
\end{tabular}
\label{tab:hyperpar_gpt}
\end{table}

\begin{table}[tb]
\centering
\caption{The hyperparameter configurations used for fine-tuning OPT-2.7B, with bold values indicating the settings applied to generate the final results.}
\begin{tabular}{llc}
\toprule
\textbf{Algorithm} & \textbf{Hyperparameters} & \textbf{Values} \\
\midrule
\multirow{4}{*}{MeZO} & Batch size & $64$ \\
                       & Learning rate & $\{1e-4, 5e-5, 1e-5, \mathbf{5e-6}, 1e-6\}$ \\
                      & $\epsilon$ & $1e-3$ \\
                      & Total Steps & $8,000$ \\
\midrule
\multirow{5}{*}{MeZO-Adam} & Batch size & $64$ \\
                      & Learning rate & $\{1e-3, 5e-4, 1e-4, \mathbf{5e-5}, 1e-5\}$ \\
                      & Betas & $(0.9, 0.999)$ \\
                      & $\epsilon$ & $1e-3$ \\
                      & Total Steps & $8,000$ \\
\midrule
\multirow{6}{*}{MeZO-SVRG} & Batch size & $64$ \\
                           & Learning rate (Full-batch) & $\{1e-4, \mathbf{5e-5}, 1e-5, 5e-6, 1e-6\}$\\
                           & Learning rate (Mini-batch) & $\{1e-5, 5e-6, \mathbf{1e-6}, 5e-7, 1e-7\}$ \\
                           & $\epsilon$ & $1e-3$ \\
                           & Frequency of Full-batch Update & $2$ \\
                           & Total Steps & $8,000$ \\
\midrule
\multirow{6}{*}{LOZO} & Batch size & $64$ \\
                      & Learning rate & $\{1e-4, 5e-5, 1e-5, \mathbf{5e-6}, 1e-6\}$ \\
                      & Rank & $\{\mathbf{2}, 4, 8\}$ \\
                      & Interval & $\{\mathbf{50}, 100\}$ \\
                      & $\epsilon$ & $1e-3$ \\
                      & Total Steps & $8,000$ \\
\midrule
\multirow{5}{*}{HiZOO} & Batch size & $64$ \\
                      & Learning rate & $\{1e-4, 5e-5, 1e-5, \mathbf{5e-6}, 1e-6\}$ \\
                      & Hessian Smoothing & $1e-8$ \\
                      & $\epsilon$ & $1e-3$ \\
                      & Total Steps & $8,000$ \\
\midrule
\multirow{6}{*}{\loren} & Batch size & $64$ \\
                      & Learning rate ($\eta$) & $\{1e-5, 5e-6, 1e-6, \mathbf{5e-7}, 1e-7\}$ \\
                      & Learning rate ($\nu$) & $1e-3$ \\
                      & Damping & $1e-1$ \\
                      & $\epsilon$ & $1e-3$ \\
                      & Total Steps & $8,000$ \\    
\bottomrule
\end{tabular}
\label{tab:hyperpar_opt}
\end{table}

\begin{table}[tb]
\centering
\caption{The hyperparameter configurations used for fine-tuning LLaMA-3 8B and OPT-13B, with bold values indicating the settings applied to generate the final results.}
\begin{tabular}{llc}
\toprule
\textbf{Algorithm} & \textbf{Hyperparameters} & \textbf{Values} \\
\midrule
\multirow{4}{*}{MeZO} & Batch size & $64$ \\
                       & Learning rate & $\{1e-5, 5e-6, \mathbf{1e-6}, 5e-7, 1e-7\}$ \\
                      & $\epsilon$ & $1e-3$ \\
                      & Total Steps & $4,000$ \\
\midrule
\multirow{5}{*}{MeZO-Adam} & Batch size & $64$ \\
                      & Learning rate & $\{1e-3, 5e-4, 1e-4, \mathbf{5e-5}, 1e-5\}$ \\
                      & Betas & $(0.9, 0.999)$ \\
                      & $\epsilon$ & $1e-3$ \\
                      & Total Steps & $4,000$ \\
\midrule
\multirow{6}{*}{LOZO} & Batch size & $64$ \\
                      & Learning rate & $\{1e-5, 5e-6, 1e-6, \mathbf{5e-7}, 1e-7\}$ \\
                      & Rank & $\{\mathbf{2}, 4, 8\}$ \\
                      & Interval & $\{\mathbf{50}, 100\}$ \\
                      & $\epsilon$ & $1e-3$ \\
                      & Total Steps & $4,000$ \\
\midrule
\multirow{5}{*}{HiZOO} & Batch size & $64$ \\
                      & Learning rate & $\{1e-5, 5e-6, \mathbf{1e-6}, 5e-7, 1e-7\}$ \\
                      & Hessian Smoothing & $1e-8$ \\
                      & $\epsilon$ & $1e-3$ \\
                      & Total Steps & $4,000$ \\
\midrule
\multirow{6}{*}{\loren} & Batch size & $64$ \\
                      & Learning rate ($\eta$) & $\{1e-5, 5e-6, 1e-6, \mathbf{5e-7}, 1e-7\}$ \\
                      & Learning rate ($\nu$) & $1e-3$ \\
                      & Damping & $1e-1$ \\
                      & $\epsilon$ & $1e-3$ \\
                      & Total Steps & $4,000$ \\    
\bottomrule
\end{tabular}
\label{tab:hyperpar_llama}
\end{table}

\end{document}